\documentclass{article}

\usepackage{arxiv}

\usepackage[utf8]{inputenc}
\usepackage[T1]{fontenc}
\usepackage{amsmath, amssymb, amsthm, amsfonts}
\usepackage{newtxtext}
\usepackage{newtxmath}
\usepackage{booktabs}
\usepackage{multirow}
\usepackage{colortbl}
\usepackage{mathtools}
\usepackage{bm}
\usepackage{graphicx}
\usepackage{nicefrac}
\usepackage{xcolor}
\usepackage{tikz}
\usetikzlibrary{arrows.meta, positioning, calc}
\usepackage{placeins}
\usepackage{microtype}
\usepackage{hyperref}
\usepackage{url}
\usepackage{natbib}

\providecommand{\keywords}[1]{%
  \begin{center}
    \small\textbf{Keywords:} #1
  \end{center}
}

\newcommand{\R}{\mathbb{R}}
\newcommand{\E}{\mathbb{E}}
\newcommand{\KL}{\mathrm{KL}}
\newcommand{\StudentT}{\mathrm{StudentT}}
\newcommand{\Normal}{\mathcal{N}}

\newcommand{\sigm}{\mathrm{sigmoid}}

\newcommand{\DeRegime}{\textsc{DeRegiME}}
\newcommand{\DeRegimeFull}{Deep Regime Mixture of Experts}
\newcommand{\PatchTST}{\textsc{PatchTST}}
\newcommand{\RevIN}{\textsc{RevIN}}

\newcommand{\second}[1]{\underline{#1}}

\newtheorem{theorem}{Theorem}
\newtheorem{proposition}[theorem]{Proposition}

\theoremstyle{definition}
\newtheorem{assumption}[theorem]{Assumption}
\theoremstyle{remark}

\title{DeRegiME: Deep Regime Mixtures for\\
Probabilistic Forecasting under Distribution Shift}
\renewcommand{\shorttitle}{DeRegiME: Deep Regime Mixtures for Probabilistic Forecasting under Distribution Shift}

\author{%
  Kieran Wood\textsuperscript{1,2},
  Stefan Zohren\textsuperscript{1},
  Stephen J. Roberts\textsuperscript{1}\\
  \textsuperscript{1}Machine Learning Research Group, University of Oxford\\
  \textsuperscript{2}Oxford-Man Institute, University of Oxford\\
  \texttt{\{kieran.wood,stefan.zohren,stephen.roberts\}@eng.ox.ac.uk}
}

\begin{document}

\maketitle

\begin{abstract}
We introduce \DeRegime{} -- \DeRegimeFull{} -- a direct multi-horizon probabilistic forecaster that separates latent uncertainty regimes from the underlying signal and softly assigns each forecast location to learned recurring regimes using a sparse variational Gaussian process (GP) whose nonstationary regime-mixing kernel and Student-$t$ likelihood combine per-regime sub-kernels and noise processes via a shared gate. This yields a single sparse-GP posterior, not a mixture of GP experts.
\DeRegime{} addresses a key limitation of neural forecasters: point forecasts discard residual uncertainty, and probabilistic heads -- whether single marginals, uninterpreted mixtures, quantile sets, or diffusion samples -- rarely expose the regime structure of the residual. Yet distribution shift in noisy heteroskedastic time series may be abrupt, gradual, or horizon-dependent and often appears in residual uncertainty rather than the conditional mean.
\DeRegime{} yields an interpretable mean--residual--noise decomposition with a direct-sum feature-space representation that anchors regimes as clusters of residual similarity whose transitions surface as implicit changepoints. The effective number of regimes is pruned by the stick-breaking gate.
We prove kernel validity and predictive-density propriety, and across ten benchmarks and three encoder grids \DeRegime{} improves negative log predictive density (NLPD) by $20.3\%$ over the strongest encoder-matched baseline, a DeepAR/GluonTS-style dynamic Student-$t$ head, with parallel gains on CRPS ($3.0\%$) and MSE ($4.7\%$). Improvements are consistent across all datasets, which span abrupt, gradual, and seasonal shifts.
Code is available at \url{https://github.com/kieranjwood/deregime}.
\end{abstract}


\begin{figure}[!h]
\centering
\includegraphics[width=0.95\linewidth]{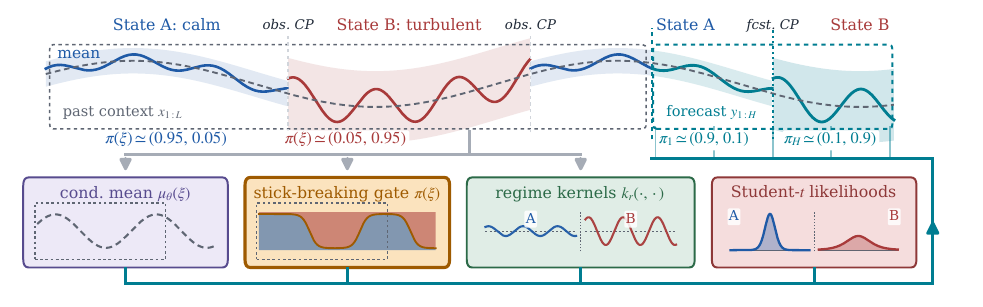}
\vspace{-3mm}
\caption{\DeRegime{} pipeline, past-context routing (grey) and predictive to the forecast (teal).}
\label{fig:concept-regime-reuse}
\vspace{-2mm}
\end{figure}

\section{Introduction}
\label{sec:intro}

Probabilistic time-series forecasting is challenging because the uncertainty state can change even when the conditional mean does not. Two past windows can imply similar expected futures but different residual behaviour: covariance, scale, or tail thickness may change abruptly at changepoints, gradually through drift, or differently across forecast horizons \citep{hamilton1989new,adams2007bocpd,garnett2010sequential}. Gaussian processes (GPs) provide principled uncertainty quantification, but standard regression typically assumes a single stationary kernel and one Gaussian noise model \citep{rasmussen2006gpml}. Neural probabilistic forecasting frameworks such as DeepAR and GluonTS attach parametric likelihood heads to encoder states \citep{salinas2020deepar,alexandrov2020gluonts}, while quantile-based forecasters such as Temporal Fusion Transformer (TFT) estimate predictive quantiles directly \citep{lim2021tft}. These heads can compress heterogeneous uncertainty states into a single marginal summary, blurring whether forecast risk arises from mean error, structured residual variation, or a high-noise regime. We frame \DeRegime{} as an output-side probabilistic mechanism: the deep forecasting architecture is held fixed across Gaussian, Student-$t$, mixture-density, quantile, single-kernel GP, and our regime-mixing heads.

\DeRegime{} (\DeRegimeFull{}) is built around the complementary decomposition
\begin{equation}
\tilde y_\xi
=
\underbrace{\mu_\theta(\xi)}_{\text{trend, drift, oscillation}}
+
\underbrace{\delta_\xi}_{\text{structured residual signal}}
+
\underbrace{\epsilon_{r_\xi,\xi}}_{\text{scale, tail, noise state}},
\qquad
\Pr(r_\xi=r\mid \xi)=\pi_r(\xi),
\label{eq:decomposition-main}
\end{equation}
where $\xi=(x,t,d)$ indexes the conditioning window, forecast horizon, and channel. The variable $r_\xi$ is an unobserved regime label for this particular forecast location, and conditional on the encoded past and horizon the neural gate assigns probabilities $\pi_r(\xi)=\Pr(r_\xi=r\mid \xi)$ over the $R_{\max}$ candidate regimes. The residual term $\epsilon_{r_\xi,\xi}$ is regime-dependent rather than the usual additive Gaussian noise. Conditional on the latent state $r$ and the GP residual $\delta_\xi$, $\epsilon_{r_\xi,\xi}$ has the regime-$r$ Student-$t$ density with scale $\sigma_r(\xi)$ and tail $\nu_r$. This notation is a generative shorthand. At prediction and training time \DeRegime{} marginalises over $r_\xi$ through the mixture likelihood rather than making a hard regime assignment. Figure~\ref{fig:concept-regime-reuse} sketches this decomposition and its recurring regime allocations. The two central objects are the gate-weighted residual covariance
\begin{equation}
K_{\mathrm{mix}}(\xi,\xi')
=
\sum_{r=1}^{R_{\max}}\pi_r(\xi)\pi_r(\xi')K_r(z_r(\xi),z_r(\xi')),
\label{eq:kmix-main}
\end{equation}
where $z_r(\xi)$ are regime-specific deep-kernel features \citep{wilson2016deep} produced by per-regime expert heads on the encoded past. Under this kernel, sparse variational inference gives a scalar marginal residual posterior at each forecast location,
\[
\delta\sim\mathcal{GP}(m_0,K_{\mathrm{mix}}),
\qquad
q(\delta_\xi\mid x)=\Normal\!\left(\delta_\xi;m_\delta(\xi),s_\delta^2(\xi)\right),
\]
where $\mathcal{GP}(m_0,K_{\mathrm{mix}})$ denotes a Gaussian-process prior over the residual function $\xi\mapsto\delta_\xi$, and $m_\delta(\xi)$ and $s_\delta^2(\xi)$ are the sparse-GP posterior mean and variance of the residual at $\xi$; Appendix~\ref{app:svgp-predictive} gives the closed forms. The predictive density marginalises both the residual $\delta_\xi$ and the unobserved regime label:
\begin{equation}
p(\tilde y_\xi\mid \xi)
=
\int q(\delta_\xi\mid x)
\sum_{r=1}^{R_{\max}}\pi_r(\xi)\,
\StudentT\!\left(\tilde y_\xi;\mu_\theta(\xi)+\delta_\xi,\sigma_r(\xi),\nu_r\right)
\,d\delta_\xi .
\label{eq:predictive-main}
\end{equation}
Here $m_0$ is the small gate-weighted residual offset, the variational posterior $q$ follows the sparse GP approximation \citep{titsias2009variational,hensman2013gaussian}, and $\sigma_r(\xi)$ is defined in Eq.~\eqref{eq:scale-main}. The gate is therefore not merely an output-mixture weight. It decides which forecast locations share residual GP structure and which scale/tail component scores the observation. The mean path captures drift, trend, and oscillatory structure, while regimes represent \emph{uncertainty states} such as calm/turbulent, light/heavy-tailed, or transient/persistent, rather than separate mean functions, with only a learned gate-weighted constant offset inside the residual GP. Since the regime gate depends jointly on the conditioning window and forecast horizon, a transient shock regime may dominate early horizons while a slower uncertainty regime persists further out. This is the key multi-horizon distinction. \DeRegime{} does not assign one latent state to an entire forecast path but learns horizon-indexed regime allocations.

The setting remains multivariate, but predictive scoring factorises into scalar marginals per horizon-channel location, matching channel-independent strong baselines like DLinear and \PatchTST{} \citep{zeng2023dlinear,nie2023patchtst}. Cross-channel sharing enters through shared encoder weights, regime parameters, and channel calibration $c_d$, so the sparse GP residual is scalar-valued and the evidence lower bound (ELBO) requires only one-dimensional Gauss--Hermite quadrature per location.

Because the same simplex weights control GP covariance, likelihood mixing, and cluster assignment, regimes are interpretable without an auxiliary clustering objective or post-hoc fitting step \citep{bishop1994mixture,jacobs1991adaptive,jordan1994hierarchical}. Gate vectors form soft clusters, and per-regime scale/tail parameters $(\tau_r,\nu_r)$ paired with channel calibration $c_d$ define a transferable set of uncertainty states (calm/volatile, light/heavy-tailed, short/long-horizon) shared across channels.

This decomposition is motivated by Bayesian changepoint modelling. A fully Bayesian approach would compare models $M_K$ with $K=0,1,2,\ldots$ changepoints or regimes through marginal likelihoods,
\begin{equation}
p(M_K\mid \mathcal{D}) \propto p(\mathcal{D}\mid M_K)p(M_K),
\label{eq:bayes-factor-motivation}
\end{equation}
but the discrete search is combinatorial and the marginal likelihoods are intractable for non-conjugate deep-kernel likelihoods. \DeRegime{} replaces this with a single differentiable class via finite regime truncation, continuous gates, and held-out predictive scores. Sharp changepoints, gradual drift, and recurrent regimes are limiting behaviours of the same gate, contrasting with run-length \citep{adams2007bocpd,garnett2010sequential} and Markov-switching \citep{hamilton1989new} formulations that condition on the current segment or require an explicit temporal state model \citep{wood2022slow}. Appendix~\ref{app:changepoint-details} details the fuller changepoint/Bayes-factor view.

The model has three coupled pieces: a mean backbone, a regime gate, and regime-specific GP features. In the reported configuration the neural gate outputs logits $\gamma_{1:R_{\max}-1}(\xi)$; their sigmoid transforms $v_r(\xi)=\sigm(\gamma_r(\xi))$ are break fractions, which are mapped to simplex weights through a finite deterministic stick-breaking map \citep{sethuraman1994constructive,ishwaran2001gibbs}:
\begin{equation}
\pi_r(\xi)=v_r(\xi)\prod_{j<r}\!\left[1-v_j(\xi)\right],
\quad r<R_{\max},
\qquad
\pi_{R_{\max}}(\xi)=\prod_{j<R_{\max}}\!\left[1-v_j(\xi)\right].
\label{eq:stick-main}
\end{equation}
Although $R_{\max}$ is fixed for computation, realised gate mass determines how many regimes are used. We summarise this through the effective regime count $R_{\mathrm{eff}}$, the number with average gate mass above $10^{-2}$ (Appendix~\ref{app:gate-details}). The same gate controls the covariance in Eq.~\eqref{eq:kmix-main} and the predictive density in Eq.~\eqref{eq:predictive-main}, distinguishing \DeRegime{} from same-backbone mixture-density heads (which mix only output distributions) and from independent mixture-of-GP experts \citep{tresp2001mixtures,rasmussen2002infinite} (which combine separately fitted local predictors).
Our contributions are:
\begin{itemize}
\item We formulate direct multi-horizon forecasting as a decomposition into a shared conditional mean, structured GP residual signal, and regime-dependent residual uncertainty, making regimes uncertainty states rather than mean clusters.
\item We propose the regime-mixing kernel for deep multi-horizon forecasting: a gate-weighted nonstationary kernel mixture that enters the GP covariance directly rather than only at the output density.
\item We establish a direct-sum feature-space representation (Proposition~\ref{prop:feature-space-main}) that anchors regimes as clusters of residual similarity, and show that finite deterministic stick-breaking over a fixed truncation $R_{\max}$ effectively prunes unused candidate regimes.
\item We couple the kernel with a horizon-indexed Student-$t$ mixture likelihood whose regimes specialise in residual scale, tail thickness, and signal-to-noise profile across forecast timescales.
\item We demonstrate on real-world benchmarks that \DeRegime{} exhibits the targeted regime-switching behaviour -- soft allocation responding to abrupt changepoints, gradual distributional drift, and horizon-dependent residual structure -- alongside aggregate gains over the strongest encoder-matched baselines.
\end{itemize}

\section{Related work}
\label{sec:related}

\paragraph{Probabilistic forecasting.}
Neural probabilistic forecasters typically pair a sequence encoder with a parametric output distribution. DeepAR \citep{salinas2020deepar} and GluonTS \citep{alexandrov2020gluonts} use Gaussian, Student-$t$, or negative-binomial likelihood heads, TFT \citep{lim2021tft} uses attention and quantile losses, and \PatchTST{} \citep{nie2023patchtst} provides a strong patch-based direct-multi-horizon backbone. Neural Processes \citep{garnelo2018conditional,garnelo2018neural} target context-set conditioning rather than long-horizon forecasting heads. \DeRegime{} is orthogonal to encoder choice: it changes the residual distribution from a single horizon-wise component to a gate-coupled regime mixture.

Mixture-density networks (MDNs) parameterise finite output mixtures \citep{bishop1994mixture}, with Student-$t$ mixtures the standard robust alternative \citep{lange1989robust,peel2000robust,fruhwirth2010bayesian}. Their mixture weights, however, do not define a GP covariance or reusable residual-similarity structure. In \DeRegime{}, the gate weights the Student-$t$ components and the kernel jointly, so mixture membership controls both density shape and prior correlation.

\paragraph{Changepoints and regime switching.}
Bayesian online changepoint detection \citep{adams2007bocpd} maintains a run-length posterior. Garnett et al.\ extend this to GP covariance functions \citep{garnett2010sequential}, and Markov-switching models \citep{hamilton1989new} provide a classical finite-state alternative. \DeRegime{} keeps the multi-regime predictive principle but replaces changepoint enumeration and latent Markov-state inference with a differentiable input-/horizon-indexed gate, so abrupt switching and gradual drift become two behaviours of the same soft allocation.

\paragraph{Mixtures of Gaussian processes and experts.}
Mixture-of-experts (MoE) models combine local predictors via input-dependent gates \citep{jacobs1991adaptive,jordan1994hierarchical,shazeer2017outrageously,fedus2022switch}, and mixtures of GPs and infinite GP experts use this for nonstationarity \citep{tresp2001mixtures,rasmussen2002infinite,volodina2020diagnostics}. \DeRegime{} differs by tying the gate to both the sparse-GP covariance and the Student-$t$ likelihood, yielding a single sparse GP posterior on a direct-sum regime feature space rather than a collection of separately normalised local predictors, and by using a finite stick-breaking truncation instead of a nonparametric posterior.

\paragraph{Sparse variational GPs and deep kernels.}
Sparse variational Gaussian processes (SVGPs) introduce inducing variables and optimise an evidence lower bound, making GP inference scalable \citep{titsias2009variational,hensman2013gaussian}. Deep kernel learning (DKL) composes kernels with learned neural features \citep{wilson2016deep}. \DeRegime{} uses both ideas, but replaces the single learned kernel with a positive semi-definite (PSD) sum over gate-weighted regime kernels and replaces the usual Gaussian observation model with a Student-$t$ regime mixture.

\section{Method}
\label{sec:method}

\paragraph{Setup.}
Given a past sequence $x\in\R^{L\times D}$, the task is to forecast a target path $y_{1:H}\in\R^{H\times D}$. \DeRegime{} is a direct multi-horizon model that emits all horizon-indexed means, gates, and GP features in one forward pass rather than recursively feeding predictions back into the input. We write $\xi=(x,t,d)$ for a forecast location consisting of the conditioning window, horizon index, and output channel when needed. All neural baselines and \DeRegime{} use reversible instance normalisation (\RevIN{}) so that the model operates in a locally normalised coordinate system and predictions are transformed back to the original scale \citep{kim2022revin}. When evaluating an original-scale density, the inverse \RevIN{} transform contributes the usual per-channel change-of-variables Jacobian, with details in Appendix~\ref{app:revin-density}.
The likelihood used in the main benchmark is a product of scalar marginal predictive densities over forecast locations $\xi$, so it should be read as a multivariate forecaster with channel-independent marginal scoring rather than a full multivariate output-density model.

\paragraph{Mean/residual/regime decomposition.}
A deep backbone (\PatchTST{}-style patch encoder in the experiments, shared across neural baselines) encodes the past into $h_\psi(x)$, and lightweight heads (linear projections over the flattened encoder output sequence in the reported runs) define
\[
\mu_\theta(\xi)=m_\theta(h_\psi(x),t,d),\qquad
\gamma_\omega(\xi)=g_\omega(h_\psi(x),t,d),\qquad
z_r(\xi)=e_{\eta,r}(h_\psi(x),t,d),
\]
where $m_\theta$ is the shared deterministic location head, $g_\omega$ produces gate logits, and $e_{\eta,r}$ are regime-specific expert heads for GP feature coordinates. The gate logits are mapped to simplex weights via Eq.~\eqref{eq:stick-main}, with implementation details and the simplex regulariser in Appendix~\ref{app:gate-details}. The asymmetry is intentional. \DeRegime{} is not an ensemble of regime-specific mean forecasts. The GP residual also includes a small gate-weighted constant prior mean $m_0(\xi)=\sum_r\pi_r(\xi)b_r$ with learned scalar offsets $b_r$ that captures regime-aligned level jumps without giving each regime a horizon-dependent decoder.

\paragraph{Finite gate and effective regimes.}
The truncation $R_{\max}$ is fixed, but active regime usage is learned through the gate mass. Because the gate is horizon-indexed, $\pi_r(x,t,d)$ can differ across forecast steps -- e.g.\ a shock regime at short horizons and a background-volatility regime at longer ones. The effective regime count $R_{\mathrm{eff}}$ is a diagnostic outcome of training. The ``stick-breaking prior'' refers to the finite deterministic stick-breaking map plus a weak Dirichlet-style penalty, not a Dirichlet-process posterior (Appendix~\ref{app:gate-details}).

\paragraph{Regime-mixing kernel.}
For each regime $r$, let $K_r$ be a PSD base kernel on the regime feature space. The mixture kernel of Eq.~\eqref{eq:kmix-main} contributes from regime $r$ only when both locations have large gate weight for $r$, so locations are similar when allocated to the same regime and close in its feature space.  This is the kernel-space analogue of mixture-of-experts clustering, and because the same $\pi_r(\xi)$ also gates the predictive density in Eq.~\eqref{eq:predictive-main}, kernel-level similarity and likelihood-level mixing are governed by one allocator rather than two parallel weight systems. As a special case, $R_{\max}=2$ with $\pi_1(\xi)=\sigm(-\beta(t-\tau))$ and $\pi_2=1-\pi_1$ recovers the smooth changepoint kernel of \citet{adams2007bocpd,garnett2010sequential}, with $\beta\to\infty$ giving a hard changepoint. \DeRegime{} therefore generalises the standard changepoint kernel to learned input and horizon-dependent gates.

In the reported \DeRegime{} runs, each $K_r$ is a squared-exponential (RBF) kernel on a learned four-dimensional regime feature, with per-regime amplitude $a_r$ and lengthscale $\ell_r$,
\begin{equation}
K_r(z,z')=a_r^2\exp\!\left(-\frac{\|z-z'\|_2^2}{2\ell_r^2}\right),
\qquad z,z'\in\R^4 .
\label{eq:rbf-base-main}
\end{equation}
The neural expert head learns the feature representation while the RBF kernel imposes a smooth prior within each regime \citep{rasmussen2006gpml}. Non-stationarity arises from the regime-specific deep features $z_r(\xi)$ \citep{wilson2016deep} and the gate-weighted mixture in Eq.~\eqref{eq:kmix-main}. Per-regime $a_r,\ell_r$ let regimes differ in signal variance and smoothness. We use an isotropic lengthscale because the feature map can rescale or rotate the regime feature and the resulting kernel is more stable in the sparse-variational setting.

\begin{proposition}[Direct-sum feature-space representation]
\label{prop:feature-space-main}
Each positive semi-definite (PSD) base kernel $K_r$ corresponds to a feature map $\phi_r:\mathcal{Z}_r\to\mathcal{H}_r$ with $K_r(z,z')=\langle\phi_r(z),\phi_r(z')\rangle_{\mathcal{H}_r}$. The map
\[
\Phi(\xi)=\bigoplus_{r=1}^{R_{\max}}\pi_r(\xi)\phi_r(z_r(\xi))
\]
takes values in $\mathcal{H}_1\oplus\cdots\oplus\mathcal{H}_{R_{\max}}$ and satisfies $K_{\mathrm{mix}}(\xi,\xi')=\langle\Phi(\xi),\Phi(\xi')\rangle$.
\end{proposition}

\begin{theorem}[Positive semi-definiteness]
\label{thm:kernel-psd-main}
If each $K_r$ is positive semi-definite and $\pi_r(\xi)\in\R$ is any real-valued gate function, then
\[
K_{\mathrm{mix}}(\xi,\xi')=\sum_{r=1}^{R_{\max}}\pi_r(\xi)\pi_r(\xi')K_r(z_r(\xi),z_r(\xi'))
\]
is positive semi-definite.
\end{theorem}

\noindent\emph{Proof.}
By Proposition~\ref{prop:feature-space-main}, $K_{\mathrm{mix}}$ is an inner product in a direct-sum Hilbert space. Every finite Gram matrix of an inner product is positive semi-definite. A matrix-form proof is in Appendix~\ref{app:proofs}.

\noindent Three consequences follow Proposition~\ref{prop:feature-space-main}: locations with disjoint-regime gates are decorrelated under the GP prior; cluster diagnostics faithfully describe kernel geometry; and the sparse-variational GP operates on a $\sim d_g R_{\max}$-dimensional direct-sum space, setting the scale for choosing $M$.

\paragraph{Student-$t$ mixture likelihood.}
The regime-mixing kernel and gate-weighted constant mean induce the marginal sparse-GP residual posterior $q(\delta_\xi\mid x)=\Normal(m_\delta(\xi),s_\delta^2(\xi))$. The Student-$t$ mixture likelihood in Eq.~\eqref{eq:predictive-main} has per-regime tail parameter $\nu_r$ and scale
\begin{equation}
\sigma_r^2(x,t,d)=(c_d\tau_r)^2+v_{\mathrm{res}}(x,t,d)+\sigma_{\mathrm{floor}}^2,
\label{eq:scale-main}
\end{equation}
where $c_d$ is a channel calibration, $\tau_r$ is a regime scale multiplier, and $v_{\mathrm{res}}$ is a residual variance correction removed in an ablation. The ratio $\rho_r(\xi)=s_\delta^2(\xi)/\sigma_r^2(\xi)$ diagnoses whether uncertainty is dominated by structured latent residual variation or observation noise. Low $\rho_r$ means observation noise dominates the predictive variance, while high $\rho_r$ means latent residual variation dominates, giving each regime an interpretable signal-to-noise profile alongside $\tau_r,\nu_r$.

The residual term $v_{\mathrm{res}}(\xi)$, shared across regimes and predicted from regime-aware features, provides an input and horizon-dependent variance floor without removing regime structure (regimes still differ in $\tau_r$, $\nu_r$, and gate weights). The principal model is the Student-$t$ mixture in Eq.~\eqref{eq:predictive-main}. Gaussian with regime-weighted variance and Gaussian-mixture variants used as ablations are detailed in Appendix~\ref{app:likelihood-variants}.

\paragraph{ELBO and training.}
Let $u=\delta(Z)$ be inducing variables for the regime-mixing GP residual, with variational posterior $q(u)=\Normal(m,S)$, and write $p_r(\tilde y\mid\delta,\xi)=\StudentT(\tilde y;\mu_\theta(\xi)+\delta,\sigma_r(\xi),\nu_r)$ for the regime-$r$ component of the predictive density in Eq.~\eqref{eq:predictive-main}. The ELBO is
\begin{equation}
\mathcal{L}_{\mathrm{ELBO}}
=
\sum_{n=1}^{N}
\E_{q(\delta_n\mid x_n)}
\left[
\log \sum_{r=1}^{R_{\max}}\pi_r(\xi_n)
p_r(\tilde y_n\mid \delta_n,\xi_n)
\right]
-
\KL(q(u)\,\|\,p(u)).
\label{eq:elbo-main}
\end{equation}
The expectation under $q(\delta_n\mid x_n)$ is evaluated by Gauss--Hermite quadrature, and the second term is the closed-form Kullback--Leibler (KL) between Gaussian $q(u)$ and prior $p(u)$. Auxiliary gate penalties are small and only guide early regime usage.

\section{Theoretical properties}
\label{sec:theory}

The PSD theorem above guarantees the kernel is valid, and Eq.~\eqref{eq:elbo-main} gives the standard SVGP training objective specialised to our setting. We state only the additional properties needed to justify the predictive density and its numerical evaluation. Proofs and qualifications are in Appendix~\ref{app:proofs}.

\paragraph{Proper density.}
Eq.~\eqref{eq:predictive-main} is a genuine probability density, not a heuristic score. To use it for held-out NLPD, the posterior predictive density must integrate to one in $y$. This follows whenever the gates lie in the probability simplex, the components $p_r(y\mid \delta,\xi)$ are proper densities, and $q(\delta_\xi\mid x)$ is a proper Gaussian. By Tonelli's theorem, the integral over $y$ can be exchanged with the nonnegative mixture and Gaussian marginalisation, giving $\sum_r \pi_r(\xi)\int q(\delta_\xi\mid x)d\delta_\xi\int p_r(y\mid\delta_\xi,\xi)dy=1$. The formal statement and proof are in Appendix~\ref{app:proofs}. This validity is the precondition for using $-\log p$ as the NLPD log score and for using the corresponding conditional log likelihood inside the ELBO in Eq.~\eqref{eq:elbo-main}.

\paragraph{Identifiability and finite sticks.}
As in standard finite mixtures, active Student-$t$ components with distinct scale--tail pairs are identifiable only up to permutation of the regime labels, and Appendix~\ref{app:identifiability} states this qualification formally. The deterministic stick-breaking map sends finite logits to the open simplex, so exact zero-weight pruning is a limiting event rather than a finite-logit guarantee. Automatic regime discovery is therefore operational: unused regimes are those whose average gate mass falls below the diagnostic threshold in $R_{\mathrm{eff}}$, while their nonzero finite-logit mass still explains the asymptotic tail caveat in Proposition~\ref{prop:tail-main}. We therefore interpret regime labels through fitted behaviour and report $R_{\mathrm{eff}}$ from thresholded average gate mass, not from an exact posterior over the number of regimes.

\paragraph{Tail control.}
The Student-$t$ mixture preserves heavy-tail expressivity: a single low-$\nu_r$ regime is sufficient to make the predictive distribution heavy-tailed.

\begin{proposition}[Tail control]
\label{prop:tail-main}
At a fixed forecast location with $\mathcal{A}=\{r:\pi_r>0\}$ and $\nu_{\mathrm{eff}}=\min_{r\in\mathcal{A}}\nu_r$, the Student-$t$ mixture predictive density has polynomial tail rate $p(y\mid x,t,d)=\Theta(|y|^{-(\nu_{\mathrm{eff}}+1)})$ as $|y|\to\infty$, and $\E[|Y|^k\mid x,t,d]<\infty$ iff $k<\nu_{\mathrm{eff}}$ \citep{embrechts1997modelling}, where $\Theta(\cdot)$ denotes Bachmann--Landau asymptotic equivalence (bounded above and below by constant multiples of the right-hand side for $|y|$ sufficiently large).
\end{proposition}

\noindent Finite-logit stick-breaking gives $\pi_r>0$ for all $r$, so $\nu_{\mathrm{eff}}=\min_r\nu_r$ in practice. The prefactor scales linearly with $\pi_r$, so a near-pruned regime has negligible mass at moderate $|y|$ yet still controls the formal tail rate. Operationally, pruning lighter-tailed regimes (high $\nu_r$) leaves $\nu_{\mathrm{eff}}$ unchanged, and only the heaviest-tailed regime determines the asymptotic tail rate. Proof and pruned-regime qualifications in Appendix~\ref{app:tail}.

\begin{assumption}[Gauss--Hermite analytic strip]
\label{ass:gh-main}
For each evaluated point, the log-mixture integrand admits a zero-free analytic continuation to a strip around the real axis and satisfies the standard Hermite-weighted growth condition.
\end{assumption}

\noindent Under Assumption~\ref{ass:gh-main}, standard Hermite-quadrature theory \citep{trefethen2008gauss} gives exponential convergence in the number of nodes $Q$. The zero-free condition is stated explicitly because a positive real-line mixture can have complex zeros off the real axis. In experiments we use $Q=20$ nodes and treat the quadrature as a deterministic numerical approximation. Appendix~\ref{app:gh-qualification} states the convergence rate explicitly and bounds the practical quadrature error in our setting.

\begin{table}[htbp]
\caption{Per-model benchmark on the \PatchTST{} grid. Standard-scaled targets, all-horizon means; cell = mean with per-seed std subscript (3 seeds); lower better; bold/underline mark best/second-best per (metric, dataset). Dynamic probabilistic heads: DeepAR/GluonTS-style Gaussian and Student-$t$ \citep{salinas2020deepar,alexandrov2020gluonts}, TFT-style quantile \citep{lim2021tft} (no density); deep-kernel GP heads: single-kernel DKL \citep{wilson2016deep}.}
\label{tab:results-patchtst}
\centering
\small
\setlength{\tabcolsep}{3pt}
\begin{tabular}{clcccccc}
\toprule
& & & \multicolumn{3}{c}{\textit{Dynamic probabilistic heads}} & \multicolumn{2}{c}{\textit{Deep-kernel GP heads}} \\
\cmidrule(lr){4-6}\cmidrule(lr){7-8}
& Dataset & \DeRegime{} & Gaussian & Student-$t$ & Quantile & RBF & RQ \\
\midrule
\multirow{10}{*}{\rotatebox[origin=c]{90}{\textbf{NLPD}}}
& ETTh1 & $\second{0.534_{\,0.006}}$ & $0.590_{\,0.001}$ & $\mathbf{0.521_{\,0.003}}$ & --- & $0.665_{\,0.001}$ & $0.662_{\,0.004}$ \\
& ETTh2 & $\mathbf{0.052_{\,0.009}}$ & $0.109_{\,0.007}$ & $\second{0.069_{\,0.010}}$ & --- & $0.405_{\,0.004}$ & $0.405_{\,0.001}$ \\
& ETTm1 & $\mathbf{0.207_{\,0.012}}$ & $0.326_{\,0.016}$ & $\second{0.209_{\,0.007}}$ & --- & $1.060_{\,0.000}$ & $1.061_{\,0.000}$ \\
& ETTm2 & $\mathbf{-0.285_{\,0.042}}$ & $-0.099_{\,0.013}$ & $\second{-0.221_{\,0.015}}$ & --- & $0.573_{\,0.020}$ & $0.880_{\,0.159}$ \\
& Exchange & $\second{-0.508_{\,0.020}}$ & $-0.464_{\,0.033}$ & $\mathbf{-0.520_{\,0.022}}$ & --- & $0.666_{\,0.000}$ & $0.672_{\,0.001}$ \\
& Electricity & $\mathbf{-0.905_{\,0.139}}$ & $-0.546_{\,0.216}$ & $\second{-0.627_{\,0.192}}$ & --- & $0.238_{\,0.016}$ & $0.253_{\,0.007}$ \\
& Traffic & $\mathbf{-0.546_{\,0.102}}$ & $-0.282_{\,0.150}$ & $\second{-0.480_{\,0.253}}$ & --- & $0.384_{\,0.004}$ & $0.394_{\,0.003}$ \\
& Nasdaq & $\mathbf{0.895_{\,0.004}}$ & $\second{0.951_{\,0.022}}$ & $0.954_{\,0.070}$ & --- & $1.071_{\,0.012}$ & $1.153_{\,0.017}$ \\
& Illness & $\mathbf{1.969_{\,0.171}}$ & $2.847_{\,0.602}$ & $\second{2.150_{\,0.140}}$ & --- & $2.294_{\,0.007}$ & $2.285_{\,0.003}$ \\
& Weather & $\second{-0.661_{\,0.008}}$ & $-0.556_{\,0.003}$ & $\mathbf{-0.665_{\,0.001}}$ & --- & $-0.281_{\,0.005}$ & $-0.267_{\,0.005}$ \\
\midrule
\multirow{10}{*}{\rotatebox[origin=c]{90}{\textbf{CRPS}}}
& ETTh1 & $\second{0.269_{\,0.001}}$ & $0.270_{\,0.001}$ & $\mathbf{0.268_{\,0.001}}$ & $0.279_{\,0.003}$ & $0.281_{\,0.000}$ & $0.280_{\,0.001}$ \\
& ETTh2 & $\mathbf{0.170_{\,0.001}}$ & $0.175_{\,0.002}$ & $\second{0.170_{\,0.000}}$ & $0.180_{\,0.004}$ & $0.242_{\,0.000}$ & $0.242_{\,0.000}$ \\
& ETTm1 & $\second{0.209_{\,0.001}}$ & $0.212_{\,0.001}$ & $\mathbf{0.209_{\,0.001}}$ & $0.210_{\,0.002}$ & $0.465_{\,0.000}$ & $0.465_{\,0.000}$ \\
& ETTm2 & $\mathbf{0.138_{\,0.001}}$ & $0.140_{\,0.002}$ & $\second{0.139_{\,0.000}}$ & $0.139_{\,0.001}$ & $0.260_{\,0.003}$ & $0.365_{\,0.102}$ \\
& Exchange & $0.086_{\,0.001}$ & $\mathbf{0.085_{\,0.001}}$ & $\second{0.086_{\,0.002}}$ & $0.087_{\,0.003}$ & $0.284_{\,0.000}$ & $0.286_{\,0.000}$ \\
& Electricity & $\mathbf{0.060_{\,0.003}}$ & $0.077_{\,0.013}$ & $0.079_{\,0.013}$ & $\second{0.061_{\,0.001}}$ & $0.134_{\,0.004}$ & $0.136_{\,0.002}$ \\
& Traffic & $\second{0.099_{\,0.001}}$ & $0.117_{\,0.024}$ & $0.109_{\,0.026}$ & $\mathbf{0.078_{\,0.002}}$ & $0.162_{\,0.001}$ & $0.162_{\,0.001}$ \\
& Nasdaq & $\mathbf{0.347_{\,0.001}}$ & $\second{0.352_{\,0.002}}$ & $0.361_{\,0.009}$ & $0.409_{\,0.024}$ & $0.395_{\,0.007}$ & $0.432_{\,0.010}$ \\
& Illness & $\second{0.893_{\,0.022}}$ & $0.908_{\,0.026}$ & $0.900_{\,0.023}$ & $\mathbf{0.870_{\,0.043}}$ & $1.311_{\,0.001}$ & $1.311_{\,0.000}$ \\
& Weather & $\second{0.110_{\,0.001}}$ & $0.111_{\,0.001}$ & $\mathbf{0.110_{\,0.000}}$ & $0.111_{\,0.000}$ & $0.125_{\,0.001}$ & $0.126_{\,0.001}$ \\
\midrule
\multirow{10}{*}{\rotatebox[origin=c]{90}{\textbf{MSE}}}
& ETTh1 & $0.322_{\,0.003}$ & $0.323_{\,0.001}$ & $0.319_{\,0.001}$ & $0.337_{\,0.007}$ & $\second{0.318_{\,0.001}}$ & $\mathbf{0.317_{\,0.000}}$ \\
& ETTh2 & $\second{0.118_{\,0.001}}$ & $0.124_{\,0.003}$ & $\mathbf{0.118_{\,0.000}}$ & $0.130_{\,0.003}$ & $0.232_{\,0.000}$ & $0.232_{\,0.000}$ \\
& ETTm1 & $\mathbf{0.217_{\,0.002}}$ & $0.219_{\,0.001}$ & $\second{0.217_{\,0.001}}$ & $0.228_{\,0.003}$ & $0.909_{\,0.001}$ & $0.910_{\,0.000}$ \\
& ETTm2 & $\mathbf{0.090_{\,0.002}}$ & $\second{0.091_{\,0.001}}$ & $0.092_{\,0.000}$ & $0.093_{\,0.001}$ & $0.207_{\,0.009}$ & $0.255_{\,0.087}$ \\
& Exchange & $\mathbf{0.027_{\,0.001}}$ & $\second{0.028_{\,0.001}}$ & $0.028_{\,0.001}$ & $0.028_{\,0.001}$ & $0.268_{\,0.000}$ & $0.271_{\,0.000}$ \\
& Electricity & $\mathbf{0.016_{\,0.001}}$ & $0.023_{\,0.007}$ & $0.025_{\,0.007}$ & $\second{0.016_{\,0.000}}$ & $0.026_{\,0.004}$ & $0.026_{\,0.002}$ \\
& Traffic & $\second{0.060_{\,0.005}}$ & $0.075_{\,0.027}$ & $0.069_{\,0.030}$ & $\mathbf{0.031_{\,0.002}}$ & $0.065_{\,0.003}$ & $0.064_{\,0.003}$ \\
& Nasdaq & $\mathbf{0.424_{\,0.004}}$ & $\second{0.455_{\,0.009}}$ & $0.464_{\,0.020}$ & $0.514_{\,0.038}$ & $0.518_{\,0.018}$ & $0.633_{\,0.042}$ \\
& Illness & $\mathbf{2.966_{\,0.055}}$ & $3.582_{\,0.259}$ & $3.266_{\,0.211}$ & $\second{3.094_{\,0.109}}$ & $5.784_{\,0.001}$ & $5.785_{\,0.001}$ \\
& Weather & $\mathbf{0.083_{\,0.001}}$ & $\second{0.084_{\,0.001}}$ & $0.084_{\,0.000}$ & $0.085_{\,0.000}$ & $0.089_{\,0.001}$ & $0.089_{\,0.001}$ \\
\bottomrule
\end{tabular}
\end{table}

\begin{table}[htbp]
\caption{Relative change (\%) of \DeRegime{} vs.\ the encoder-matched Student-$t$ head (strongest baseline by NLPD). Encoder fixed; negative favours \DeRegime{}; subscripts give delta-method uncertainty propagated from per-seed stds (3 seeds). Patch./DLin./TimeM.\ denote the \PatchTST{}/DLinear/TimeMixer encoder grids; per-grid breakdowns in Tables~\ref{tab:results-patchtst},~\ref{tab:results-dlinear},~\ref{tab:results-timemixer}.}
\label{tab:relative-summary}
\centering
\setlength{\tabcolsep}{2pt}
\renewcommand{\arraystretch}{0.95}
\resizebox{\textwidth}{!}{%
\begin{tabular}{clrrrrrrrrrr!{\color{gray!50}\vrule}r}
\toprule
&  & ETTh1 & ETTh2 & ETTm1 & ETTm2 & Exch. & Elec. & Traffic & Nasdaq & Illness & Weather & Avg. \\
\midrule
\multirow{4}{*}{\rotatebox[origin=c]{90}{\textbf{$\Delta$NLPD}}}
& Patch. & $+2.5_{\,1.3}$ & $-24.6_{\,17.0}$ & $-1.2_{\,6.6}$ & $-28.7_{\,20.9}$ & $+2.3_{\,5.6}$ & $-44.4_{\,49.4}$ & $-13.7_{\,63.6}$ & $-6.2_{\,6.9}$ & $-8.4_{\,9.9}$ & $+0.7_{\,1.2}$ & $\mathbf{-12.2_{\,8.6}}$ \\
& DLin. & $-11.5_{\,2.6}$ & $-11.8_{\,12.5}$ & $-29.4_{\,5.4}$ & $-93.3_{\,27.2}$ & $-2.9_{\,1.2}$ & $-9.5_{\,1.0}$ & $-41.3_{\,4.4}$ & $-2.9_{\,0.7}$ & $+1.9_{\,2.2}$ & $-10.8_{\,1.7}$ & $\mathbf{-21.2_{\,3.1}}$ \\
& TimeM. & $+1.2_{\,1.7}$ & $-47.0_{\,8.7}$ & $-7.3_{\,10.6}$ & $-183.2_{\,115.8}$ & $-3.1_{\,6.1}$ & $+1.9_{\,18.1}$ & $+19.8_{\,16.1}$ & $+5.0_{\,3.7}$ & $-12.0_{\,10.9}$ & $-51.6_{\,13.6}$ & $\mathbf{-27.6_{\,12.1}}$ \\
\arrayrulecolor{gray!50}\cmidrule(lr){2-13}\arrayrulecolor{black}
& \textbf{Mean} & $\mathbf{-2.6_{\,1.1}}$ & $\mathbf{-27.8_{\,7.6}}$ & $\mathbf{-12.7_{\,4.5}}$ & $\mathbf{-101.8_{\,40.3}}$ & $\mathbf{-1.2_{\,2.8}}$ & $\mathbf{-17.4_{\,17.6}}$ & $\mathbf{-11.7_{\,21.9}}$ & $\mathbf{-1.4_{\,2.6}}$ & $\mathbf{-6.1_{\,5.0}}$ & $\mathbf{-20.6_{\,4.6}}$ & $\boxed{\mathbf{-20.3_{\,5.0}}}$ \\
\midrule
\multirow{4}{*}{\rotatebox[origin=c]{90}{\textbf{$\Delta$CRPS}}}
& Patch. & $+0.5_{\,0.5}$ & $-0.1_{\,0.6}$ & $+0.2_{\,0.7}$ & $-0.6_{\,0.7}$ & $+0.0_{\,2.6}$ & $-23.6_{\,13.1}$ & $-8.9_{\,21.7}$ & $-3.9_{\,2.4}$ & $-0.8_{\,3.5}$ & $+0.3_{\,0.9}$ & $\mathbf{-3.7_{\,2.6}}$ \\
& DLin. & $-4.3_{\,0.8}$ & $-0.3_{\,0.0}$ & $-8.3_{\,0.8}$ & $-2.3_{\,1.4}$ & $-0.9_{\,0.0}$ & $-9.2_{\,0.0}$ & $-18.5_{\,1.7}$ & $-1.3_{\,0.0}$ & $+2.3_{\,1.6}$ & $-9.0_{\,0.8}$ & $\mathbf{-5.2_{\,0.3}}$ \\
& TimeM. & $+0.3_{\,0.5}$ & $-0.7_{\,0.0}$ & $-0.8_{\,1.9}$ & $-4.0_{\,0.7}$ & $-1.6_{\,3.4}$ & $-2.1_{\,7.2}$ & $+16.0_{\,1.8}$ & $+1.3_{\,1.8}$ & $+0.5_{\,2.2}$ & $-9.0_{\,1.1}$ & $\mathbf{+0.0_{\,0.9}}$ \\
\arrayrulecolor{gray!50}\cmidrule(lr){2-13}\arrayrulecolor{black}
& \textbf{Mean} & $\mathbf{-1.2_{\,0.4}}$ & $\mathbf{-0.4_{\,0.2}}$ & $\mathbf{-2.9_{\,0.7}}$ & $\mathbf{-2.3_{\,0.6}}$ & $\mathbf{-0.8_{\,1.4}}$ & $\mathbf{-11.6_{\,5.0}}$ & $\mathbf{-3.8_{\,7.3}}$ & $\mathbf{-1.3_{\,1.0}}$ & $\mathbf{+0.7_{\,1.5}}$ & $\mathbf{-5.9_{\,0.5}}$ & $\boxed{\mathbf{-3.0_{\,0.9}}}$ \\
\midrule
\multirow{4}{*}{\rotatebox[origin=c]{90}{\textbf{$\Delta$MSE}}}
& Patch. & $+0.9_{\,1.0}$ & $+0.3_{\,0.8}$ & $-0.1_{\,1.0}$ & $-1.3_{\,2.2}$ & $-1.4_{\,5.0}$ & $-36.2_{\,18.4}$ & $-12.6_{\,38.5}$ & $-8.7_{\,4.0}$ & $-9.2_{\,6.1}$ & $-1.0_{\,1.2}$ & $\mathbf{-6.9_{\,4.4}}$ \\
& DLin. & $-0.7_{\,0.9}$ & $+1.0_{\,0.0}$ & $-6.6_{\,0.5}$ & $-1.6_{\,1.1}$ & $-0.2_{\,0.0}$ & $-23.2_{\,4.8}$ & $-32.1_{\,1.2}$ & $-1.9_{\,1.1}$ & $-0.8_{\,1.0}$ & $-22.0_{\,1.1}$ & $\mathbf{-8.8_{\,0.5}}$ \\
& TimeM. & $+0.0_{\,0.4}$ & $+0.4_{\,0.8}$ & $-1.6_{\,3.2}$ & $-1.9_{\,1.5}$ & $-1.3_{\,3.7}$ & $-9.5_{\,5.2}$ & $+38.0_{\,17.1}$ & $+2.6_{\,3.4}$ & $-1.6_{\,2.1}$ & $-8.4_{\,1.4}$ & $\mathbf{+1.7_{\,1.9}}$ \\
\arrayrulecolor{gray!50}\cmidrule(lr){2-13}\arrayrulecolor{black}
& \textbf{Mean} & $\mathbf{+0.1_{\,0.5}}$ & $\mathbf{+0.6_{\,0.4}}$ & $\mathbf{-2.8_{\,1.1}}$ & $\mathbf{-1.6_{\,0.9}}$ & $\mathbf{-1.0_{\,2.1}}$ & $\mathbf{-23.0_{\,6.6}}$ & $\mathbf{-2.3_{\,14.0}}$ & $\mathbf{-2.7_{\,1.8}}$ & $\mathbf{-3.8_{\,2.2}}$ & $\mathbf{-10.5_{\,0.7}}$ & $\boxed{\mathbf{-4.7_{\,1.6}}}$ \\
\bottomrule
\end{tabular}%
}
\end{table}

\paragraph{Ablations.}
We run PatchTST-only ablations and head-level stress tests on all ten datasets. All entries are measured against the reported \DeRegime{} model with the same \PatchTST{} encoder, and use the arithmetic average of seed-level test metrics.

\begin{table}[t]
\caption{PatchTST ablations and stress tests. Entries are average relative metric change (\%) against the reported \DeRegime{} model on all ten datasets; subscripts give one-sigma delta-method uncertainty propagated from seed-level variability and averaged across datasets. Negative values favour the ablation. All variants use 3 seeds per dataset. Variant descriptions in Appendix~\ref{app:ablation-details}.}
\label{tab:ablations-patchtst}
\centering
\small
\setlength{\tabcolsep}{4pt}
\begin{minipage}[t]{0.48\linewidth}
\centering
\textit{Architectural ablations}\\[2pt]
\begin{tabular}{lrrr}
\toprule
Variant & $\Delta$ NLPD & $\Delta$ CRPS & $\Delta$ MSE \\
\midrule
Softmax gate     & $+3.4_{\,4.3}$  & $+0.9_{\,0.9}$  & $+0.6_{\,2.2}$  \\
No deep mean     & $+29.0_{\,5.9}$ & $+11.6_{\,2.8}$ & $+23.5_{\,9.5}$ \\
No resid.\ var.  & $+0.1_{\,4.6}$  & $+0.5_{\,1.0}$  & $+0.5_{\,2.0}$  \\
Shared lik.      & $+8.3_{\,4.1}$  & $+2.3_{\,1.2}$  & $-1.0_{\,1.2}$  \\
Single kernel    & $+73.6_{\,5.1}$ & $+23.8_{\,1.2}$ & $+9.5_{\,2.2}$  \\
\bottomrule
\end{tabular}
\end{minipage}%
\hfill
\begin{minipage}[t]{0.48\linewidth}
\centering
\textit{Methodological controls}\\[2pt]
\begin{tabular}{lrrr}
\toprule
Variant & $\Delta$ NLPD & $\Delta$ CRPS & $\Delta$ MSE \\
\midrule
\DeRegime{} ($\mathcal{N}$)      & $+43.7_{\,5.4}$ & $+6.4_{\,3.3}$  & $+8.3_{\,1.7}$ \\
\DeRegime{} ($\mathcal{N}$ mix)  & $+10.2_{\,4.0}$ & $+1.5_{\,0.8}$  & $-1.6_{\,1.2}$  \\
MDN ($\mathcal{N}$)              & $+23.8_{\,5.8}$ & $+4.8_{\,2.5}$  & $+10.5_{\,4.6}$ \\
MDN ($t$)                        & $+23.8_{\,7.1}$ & $+4.3_{\,3.1}$  & $+9.1_{\,5.4}$ \\
\bottomrule
\end{tabular}
\end{minipage}
\end{table}

\section{Experiments}
\label{sec:experiments}

\paragraph{Setup and baselines.}
We use the standard long-horizon forecasting benchmarks \citep{zhou2021informer,wu2021autoformer,nie2023patchtst,qiu2024tfb} (Electricity Transformer Temperature (ETT) family, electricity, traffic, weather, exchange, illness), with electricity and traffic as univariate aggregates and exchange and weather as continuous-channel subsets, and add an extended Nasdaq daily-close series from 1990-01-02 to 2024-12-30, whose held-out test period (2017-12-26 to 2024-12-30) covers COVID-19, the inflation shock, and multiple market regimes. All density models share the encoder, \RevIN{} preprocessing, chronological split, and horizon aggregation. We report three encoder grids --- \PatchTST{} \citep{nie2023patchtst}, DLinear \citep{zeng2023dlinear}, and TimeMixer \citep{wang2024timemixer} --- with the headline tables using \PatchTST{}. We report MSE, continuous ranked probability score (CRPS), and proper NLPD on standard-scaled targets, averaged across horizons and arithmetically across seeds (quantile regression has no native density and is excluded from NLPD). Training used a mix of NVIDIA A100 80GB and B200-class GPUs with details provided in Appendix~\ref{app:compute}. Full experiment protocol is detailed in Appendix~\ref{app:experiment-details}. Code, configuration files, dataset-preparation scripts, and reproduction instructions are available at \url{https://github.com/kieranjwood/deregime}. The baseline matrix (Appendix~\ref{app:baseline-matrix}) isolates the source of gains across single dynamic distributions (Gaussian, Student-$t$), distribution-free intervals (quantile), and GP without regime mixing (DKL). MDN heads appear as ablations since they retain \DeRegime{}'s mixture-density component but drop the kernel-coupled GP residual.

\paragraph{Findings.}
The detailed \PatchTST{} grid is shown in Table~\ref{tab:results-patchtst}. Against the strongest encoder-matched baseline (Student-$t$ head, fixed encoder), \DeRegime{} reduces three-grid average NLPD/CRPS/MSE by $20.3\%/3.0\%/4.7\%$ (Table~\ref{tab:relative-summary}), with consistent gains across benchmarks spanning abrupt, gradual, and seasonal regime shifts. Ablations (Table~\ref{tab:ablations-patchtst}) show the regime-mixing kernel and Student-$t$ tails carry most of the margin. DKL baselines without regime mixing and single-kernel collapse ($R{=}1$) lose substantially, while same-backbone MDN heads underperform the kernel-coupled mixture. Figure~\ref{fig:regime-diagnostics-main} surfaces four targeted regime-switching behaviours, with the remaining datasets studied in Appendix~\ref{app:diagnostics}. On Illness, outbreak windows coincide with activation of a regime whose parameter diagnostics have low $\nu_r$, linking the gate to heavier tails. On GBP exchange, an abrupt gate transition tracks the onset of the eurozone debt crisis. Aggregate Traffic shows recurring demand-state cycles encoded by only a few of the $R_{\max}=16$ candidate regimes, illustrating the stick-breaking pruning property. ETTm1 exhibits abrupt within-window regime changes whose allocation also differs across forecast horizons --- the horizon-indexed gate behaviour that distinguishes \DeRegime{} from a Markov-state model.
\begin{figure}[t]
\centering
\includegraphics[width=\linewidth]{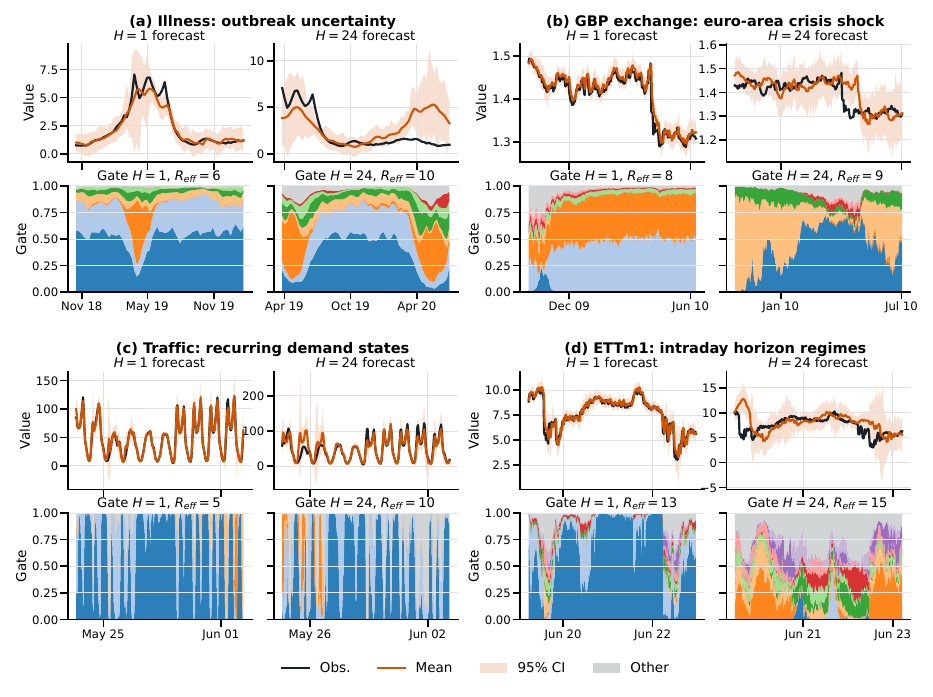}
\caption{Representative \DeRegime{} regime diagnostics: outbreak (Illness), eurozone-debt-crisis shift (GBP), few-state reuse (aggregate Traffic), and a horizon-indexed intraday transition (ETTm1). Gate colours per panel are local display ranks by average mass, with grey marking ``Other''.}
\label{fig:regime-diagnostics-main}
\end{figure}

\section{Discussion}
\label{sec:discussion}

\DeRegime{} targets direct multi-horizon probabilistic forecasting in heteroskedastic series with recurring uncertainty states. Gains are consistent across the benchmark suite, which spans abrupt, gradual, and seasonal regime shifts. The decomposition into interpretable regimes helps communicate predictive risk in real-world applications. Training cost is dominated by the sparse-GP $\mathcal{O}(M^3)$ Cholesky over $M$ inducing points (independent of dataset size $N$) and an $R_{\max}$-fold scaling from the regime-mixing kernel. Because the stick-breaking curriculum and gate-entropy regulariser concentrate mass on $R_{\mathrm{eff}}\ll R_{\max}$ regimes (Appendix~\ref{app:gate-details}), top-$k$ or low-mass-regime pruning is a natural deployment optimisation, though we do not benchmark its impact on predictive accuracy here. Diffusion, score-based, and flow forecasters use different scoring protocols (Appendix~\ref{app:generative-not-included}) and are not benchmarked.

Regimes are identifiable only up to permutation, so cross-seed comparison should use relative scales $\tau_r/\tau_{r'}$ and gate trajectories rather than absolute labels. Gauss--Hermite convergence requires a zero-free analytic-strip assumption (Appendix~\ref{app:gh-qualification}), though in practice the quadrature is a stable approximation. Finite-logit stick-breaking gives $\pi_r>0$, so a near-pruned regime can still set the asymptotic tail rate (Proposition~\ref{prop:tail-main}). The reported multivariate setting uses channel-independent marginal densities. Channels share encoders, regime parameters, and calibration scales, but the model is not a full joint predictive density over cross-channel residual covariance. Avenues for future work include per-regime base kernels beyond isotropic RBF, which are direct swap-ins since Theorem~\ref{thm:kernel-psd-main} holds for any PSD kernel, and pairing the regime-mixing head with channel-mixing backbones \citep{liu2023itransformer} or time-series foundation models \citep{ansari2024chronos,goswami2024moment,das2024timesfm,woo2024moirai} to amortise the regime parameters $(\tau_r,\nu_r)$ across datasets, especially where they are weakly identified.

\section{Acknowledgements}

KW would like to thank the Oxford-Man Institute of Quantitative Finance for its generous support.
SR would like to thank the U.K. Royal Academy of Engineering.

\bibliographystyle{plainnat}
{\small
\bibliography{references}

\begin{thebibliography}{50}
\providecommand{\natexlab}[1]{#1}
\providecommand{\url}[1]{\texttt{#1}}
\expandafter\ifx\csname urlstyle\endcsname\relax
  \providecommand{\doi}[1]{doi: #1}\else
  \providecommand{\doi}{doi: \begingroup \urlstyle{rm}\Url}\fi

\bibitem[Adams and MacKay(2007)]{adams2007bocpd}
Ryan~P. Adams and David J.~C. MacKay.
\newblock Bayesian online changepoint detection, 2007.
\newblock arXiv:0710.3742.

\bibitem[Alexandrov et~al.(2020)Alexandrov, Benidis, Bohlke-Schneider, Flunkert, Gasthaus, Januschowski, Maddix, Rangapuram, Salinas, Schulz, Stella, T{\"u}rkmen, and Wang]{alexandrov2020gluonts}
Alexander Alexandrov, Konstantinos Benidis, Michael Bohlke-Schneider, Valentin Flunkert, Jan Gasthaus, Tim Januschowski, Danielle~C. Maddix, Syama Rangapuram, David Salinas, Jasper Schulz, Lorenzo Stella, Ali~Caner T{\"u}rkmen, and Yuyang Wang.
\newblock {GluonTS}: Probabilistic and neural time series modeling in python.
\newblock \emph{Journal of Machine Learning Research}, 21\penalty0 (116):\penalty0 1--6, 2020.

\bibitem[Ansari et~al.(2024)Ansari, Stella, Turkmen, Zhang, Mercado, Shen, Shchur, Rangapuram, Arango, Kapoor, Zschiegner, Maddix, Mahoney, Torkkola, Wilson, Bohlke-Schneider, and Wang]{ansari2024chronos}
Abdul~Fatir Ansari, Lorenzo Stella, Caner Turkmen, Xiyuan Zhang, Pedro Mercado, Huibin Shen, Oleksandr Shchur, Syama~Sundar Rangapuram, Sebastian~Pineda Arango, Shubham Kapoor, Jasper Zschiegner, Danielle~C. Maddix, Michael~W. Mahoney, Kari Torkkola, Andrew~Gordon Wilson, Michael Bohlke-Schneider, and Yuyang Wang.
\newblock Chronos: Learning the language of time series.
\newblock \emph{Transactions on Machine Learning Research}, 2024.

\bibitem[Bishop(1994)]{bishop1994mixture}
Christopher~M. Bishop.
\newblock Mixture density networks.
\newblock Technical Report NCRG/94/004, Aston University, 1994.

\bibitem[Das et~al.(2024)Das, Kong, Sen, and Zhou]{das2024timesfm}
Abhimanyu Das, Weihao Kong, Rajat Sen, and Yichen Zhou.
\newblock A decoder-only foundation model for time-series forecasting.
\newblock In \emph{Proceedings of the International Conference on Machine Learning}, 2024.

\bibitem[Embrechts et~al.(1997)Embrechts, Kl{\"u}ppelberg, and Mikosch]{embrechts1997modelling}
Paul Embrechts, Claudia Kl{\"u}ppelberg, and Thomas Mikosch.
\newblock \emph{Modelling Extremal Events for Insurance and Finance}, volume~33 of \emph{Applications of Mathematics}.
\newblock Springer, 1997.

\bibitem[Fedus et~al.(2022)Fedus, Zoph, and Shazeer]{fedus2022switch}
William Fedus, Barret Zoph, and Noam Shazeer.
\newblock Switch transformers: Scaling to trillion parameter models with simple and efficient sparsity.
\newblock \emph{Journal of Machine Learning Research}, 23\penalty0 (120):\penalty0 1--39, 2022.

\bibitem[Fr{\"u}hwirth-Schnatter and Pyne(2010)]{fruhwirth2010bayesian}
Sylvia Fr{\"u}hwirth-Schnatter and Saumyadipta Pyne.
\newblock Bayesian inference for finite mixtures of univariate and multivariate skew-normal and skew-$t$ distributions.
\newblock \emph{Biostatistics}, 11\penalty0 (2):\penalty0 317--336, 2010.

\bibitem[Garnelo et~al.(2018{\natexlab{a}})Garnelo, Rosenbaum, Maddison, Ramalho, Saxton, Shanahan, Teh, Rezende, and Eslami]{garnelo2018conditional}
Marta Garnelo, Dan Rosenbaum, Christopher Maddison, Tiago Ramalho, David Saxton, Murray Shanahan, Yee~Whye Teh, Danilo Rezende, and S.~M.~Ali Eslami.
\newblock Conditional neural processes.
\newblock In \emph{Proceedings of the International Conference on Machine Learning}, 2018{\natexlab{a}}.

\bibitem[Garnelo et~al.(2018{\natexlab{b}})Garnelo, Schwarz, Rosenbaum, Viola, Rezende, Eslami, and Teh]{garnelo2018neural}
Marta Garnelo, Jonathan Schwarz, Dan Rosenbaum, Fabio Viola, Danilo~J. Rezende, S.~M.~Ali Eslami, and Yee~Whye Teh.
\newblock Neural processes, 2018{\natexlab{b}}.
\newblock arXiv:1807.01622.

\bibitem[Garnett et~al.(2010)Garnett, Osborne, Reece, Rogers, and Roberts]{garnett2010sequential}
Roman Garnett, Michael~A. Osborne, Steven Reece, Alex Rogers, and Stephen~J. Roberts.
\newblock Sequential {Bayesian} prediction in the presence of changepoints and faults.
\newblock \emph{The Computer Journal}, 53\penalty0 (9):\penalty0 1430--1446, 2010.

\bibitem[Gneiting and Raftery(2007)]{gneiting2007proper}
Tilmann Gneiting and Adrian~E. Raftery.
\newblock Strictly proper scoring rules, prediction, and estimation.
\newblock \emph{Journal of the American Statistical Association}, 102\penalty0 (477):\penalty0 359--378, 2007.

\bibitem[Goswami et~al.(2024)Goswami, Szafer, Choudhry, Cai, Li, and Dubrawski]{goswami2024moment}
Mononito Goswami, Konrad Szafer, Arjun Choudhry, Yifu Cai, Shuo Li, and Artur Dubrawski.
\newblock {MOMENT}: A family of open time-series foundation models.
\newblock In \emph{Proceedings of the International Conference on Machine Learning}, 2024.

\bibitem[Hamilton(1989)]{hamilton1989new}
James~D. Hamilton.
\newblock A new approach to the economic analysis of nonstationary time series and the business cycle.
\newblock \emph{Econometrica}, 57\penalty0 (2):\penalty0 357--384, 1989.

\bibitem[Hensman et~al.(2013)Hensman, Fusi, and Lawrence]{hensman2013gaussian}
James Hensman, Nicolo Fusi, and Neil~D. Lawrence.
\newblock Gaussian processes for big data.
\newblock In \emph{Proceedings of the Conference on Uncertainty in Artificial Intelligence}, 2013.

\bibitem[Holzmann et~al.(2006)Holzmann, Munk, and Gneiting]{holzmann2006identifiability}
Hajo Holzmann, Axel Munk, and Tilmann Gneiting.
\newblock Identifiability of finite mixtures of elliptical distributions.
\newblock \emph{Scandinavian Journal of Statistics}, 33\penalty0 (4):\penalty0 753--763, 2006.

\bibitem[Ishwaran and James(2001)]{ishwaran2001gibbs}
Hemant Ishwaran and Lancelot~F. James.
\newblock Gibbs sampling methods for stick-breaking priors.
\newblock \emph{Journal of the American Statistical Association}, 96\penalty0 (453):\penalty0 161--173, 2001.

\bibitem[Jacobs et~al.(1991)Jacobs, Jordan, Nowlan, and Hinton]{jacobs1991adaptive}
Robert~A. Jacobs, Michael~I. Jordan, Steven~J. Nowlan, and Geoffrey~E. Hinton.
\newblock Adaptive mixtures of local experts.
\newblock \emph{Neural Computation}, 3\penalty0 (1):\penalty0 79--87, 1991.

\bibitem[Jordan and Jacobs(1994)]{jordan1994hierarchical}
Michael~I. Jordan and Robert~A. Jacobs.
\newblock Hierarchical mixtures of experts and the {EM} algorithm.
\newblock \emph{Neural Computation}, 6\penalty0 (2):\penalty0 181--214, 1994.

\bibitem[Kim et~al.(2022)Kim, Kim, Tae, Park, Choi, and Choo]{kim2022revin}
Taesung Kim, Jinhee Kim, Yunwon Tae, Cheonbok Park, Jang-Ho Choi, and Jaegul Choo.
\newblock Reversible instance normalization for accurate time-series forecasting against distribution shift.
\newblock In \emph{International Conference on Learning Representations}, 2022.

\bibitem[Kollovieh et~al.(2023)Kollovieh, Ansari, Bohlke-Schneider, Zschiegner, Wang, and Wang]{kollovieh2023predict}
Marcel Kollovieh, Abdul~Fatir Ansari, Michael Bohlke-Schneider, Jasper Zschiegner, Hao Wang, and Yuyang Wang.
\newblock Predict, refine, synthesize: Self-guiding diffusion models for probabilistic time series forecasting.
\newblock In \emph{Advances in Neural Information Processing Systems}, 2023.

\bibitem[Lange et~al.(1989)Lange, Little, and Taylor]{lange1989robust}
Kenneth~L. Lange, Roderick J.~A. Little, and Jeremy M.~G. Taylor.
\newblock Robust statistical modeling using the $t$ distribution.
\newblock \emph{Journal of the American Statistical Association}, 84\penalty0 (408):\penalty0 881--896, 1989.

\bibitem[Lim et~al.(2021)Lim, Arik, Loeff, and Pfister]{lim2021tft}
Bryan Lim, Sercan~{\"O}. Arik, Nicolas Loeff, and Tomas Pfister.
\newblock Temporal fusion transformers for interpretable multi-horizon time series forecasting.
\newblock \emph{International Journal of Forecasting}, 37\penalty0 (4):\penalty0 1748--1764, 2021.

\bibitem[Liu et~al.(2023)Liu, Hu, Zhang, Wu, Wang, Ma, and Long]{liu2023itransformer}
Yong Liu, Tengge Hu, Haoran Zhang, Haixu Wu, Shiyu Wang, Lintao Ma, and Mingsheng Long.
\newblock i{T}ransformer: Inverted transformers are effective for time series forecasting, 2023.
\newblock arXiv:2310.06625.

\bibitem[McLachlan and Basford(1988)]{mclachlan1988mixture}
Geoffrey~J. McLachlan and Kaye~E. Basford.
\newblock \emph{Mixture Models: Inference and Applications to Clustering}.
\newblock Marcel Dekker, 1988.

\bibitem[McLachlan and Peel(2000)]{mclachlan2000finite}
Geoffrey~J. McLachlan and David Peel.
\newblock \emph{Finite Mixture Models}.
\newblock Wiley, 2000.

\bibitem[Nie et~al.(2023)Nie, Nguyen, Sinthong, and Kalagnanam]{nie2023patchtst}
Yuqi Nie, Nam~H. Nguyen, Phanwadee Sinthong, and Jayant Kalagnanam.
\newblock A time series is worth 64 words: Long-term forecasting with transformers.
\newblock In \emph{International Conference on Learning Representations}, 2023.

\bibitem[Peel and McLachlan(2000)]{peel2000robust}
David Peel and Geoffrey~J. McLachlan.
\newblock Robust mixture modelling using the $t$ distribution.
\newblock \emph{Statistics and Computing}, 10\penalty0 (4):\penalty0 339--348, 2000.

\bibitem[Qiu et~al.(2024)Qiu, Hu, Zhou, Wu, Du, Zhang, Guo, Zhou, Jensen, Sheng, and Yang]{qiu2024tfb}
Xiangfei Qiu, Jilin Hu, Lekui Zhou, Xingjian Wu, Junyang Du, Buang Zhang, Chenjuan Guo, Aoying Zhou, Christian~S. Jensen, Zhenli Sheng, and Bin Yang.
\newblock {TFB}: Towards comprehensive and fair benchmarking of time series forecasting methods.
\newblock \emph{Proceedings of the VLDB Endowment}, 17\penalty0 (9):\penalty0 2363--2377, 2024.
\newblock \doi{10.14778/3665844.3665863}.

\bibitem[Rasmussen and Ghahramani(2002)]{rasmussen2002infinite}
Carl~E. Rasmussen and Zoubin Ghahramani.
\newblock Infinite mixtures of gaussian process experts.
\newblock In \emph{Advances in Neural Information Processing Systems}, 2002.

\bibitem[Rasmussen and Williams(2006)]{rasmussen2006gpml}
Carl~E. Rasmussen and Christopher K.~I. Williams.
\newblock \emph{Gaussian Processes for Machine Learning}.
\newblock MIT Press, 2006.

\bibitem[Rasul et~al.(2020)Rasul, Sheikh, Schuster, Bergmann, and Vollgraf]{rasul2020multivariate}
Kashif Rasul, Abdul-Saboor Sheikh, Ingmar Schuster, Urs Bergmann, and Roland Vollgraf.
\newblock Multivariate probabilistic time series forecasting via conditioned normalizing flows.
\newblock In \emph{International Conference on Learning Representations}, 2020.

\bibitem[Rasul et~al.(2021)Rasul, Seward, Schuster, and Vollgraf]{rasul2021autoregressive}
Kashif Rasul, Calvin Seward, Ingmar Schuster, and Roland Vollgraf.
\newblock Autoregressive denoising diffusion models for multivariate probabilistic time series forecasting.
\newblock In \emph{Proceedings of the International Conference on Machine Learning}, 2021.

\bibitem[Salinas et~al.(2020)Salinas, Flunkert, Gasthaus, and Januschowski]{salinas2020deepar}
David Salinas, Valentin Flunkert, Jan Gasthaus, and Tim Januschowski.
\newblock {DeepAR}: Probabilistic forecasting with autoregressive recurrent networks.
\newblock \emph{International Journal of Forecasting}, 36\penalty0 (3):\penalty0 1181--1191, 2020.

\bibitem[Sethuraman(1994)]{sethuraman1994constructive}
Jayaram Sethuraman.
\newblock A constructive definition of {Dirichlet} priors.
\newblock \emph{Statistica Sinica}, 4\penalty0 (2):\penalty0 639--650, 1994.

\bibitem[Shazeer et~al.(2017)Shazeer, Mirhoseini, Maziarz, Davis, Le, Hinton, and Dean]{shazeer2017outrageously}
Noam Shazeer, Azalia Mirhoseini, Krzysztof Maziarz, Andy Davis, Quoc Le, Geoffrey Hinton, and Jeff Dean.
\newblock Outrageously large neural networks: The sparsely-gated mixture-of-experts layer.
\newblock In \emph{International Conference on Learning Representations}, 2017.

\bibitem[Teicher(1963)]{teicher1963identifiability}
Henry Teicher.
\newblock Identifiability of finite mixtures.
\newblock \emph{Annals of Mathematical Statistics}, 34\penalty0 (4):\penalty0 1265--1269, 1963.

\bibitem[Titsias(2009)]{titsias2009variational}
Michalis Titsias.
\newblock Variational learning of inducing variables in sparse gaussian processes.
\newblock In \emph{Proceedings of the International Conference on Artificial Intelligence and Statistics}, 2009.

\bibitem[Trefethen(2008)]{trefethen2008gauss}
Lloyd~N. Trefethen.
\newblock Is gauss quadrature better than clenshaw--curtis?
\newblock \emph{SIAM Review}, 50\penalty0 (1):\penalty0 67--87, 2008.

\bibitem[Tresp(2001)]{tresp2001mixtures}
Volker Tresp.
\newblock Mixtures of gaussian processes.
\newblock In \emph{Advances in Neural Information Processing Systems}, 2001.

\bibitem[Volodina and Williamson(2020)]{volodina2020diagnostics}
Victoria Volodina and Daniel Williamson.
\newblock Diagnostics-driven nonstationary emulators using kernel mixtures.
\newblock \emph{SIAM/ASA Journal on Uncertainty Quantification}, 8\penalty0 (1):\penalty0 1--26, 2020.

\bibitem[Wang et~al.(2024)Wang, Wu, Shi, Hu, Luo, Ma, Zhang, and Zhou]{wang2024timemixer}
Shiyu Wang, Haixu Wu, Xiaoming Shi, Tengge Hu, Huakun Luo, Lintao Ma, James~Y. Zhang, and Jun Zhou.
\newblock {TimeMixer}: Decomposable multiscale mixing for time series forecasting.
\newblock In \emph{Proceedings of the International Conference on Learning Representations}, 2024.

\bibitem[Wilson et~al.(2016)Wilson, Hu, Salakhutdinov, and Xing]{wilson2016deep}
Andrew~G. Wilson, Zhiting Hu, Ruslan Salakhutdinov, and Eric~P. Xing.
\newblock Deep kernel learning.
\newblock In \emph{Proceedings of the International Conference on Artificial Intelligence and Statistics}, 2016.

\bibitem[Woo et~al.(2024)Woo, Liu, Kumar, Xiong, Savarese, and Sahoo]{woo2024moirai}
Gerald Woo, Chenghao Liu, Akshat Kumar, Caiming Xiong, Silvio Savarese, and Doyen Sahoo.
\newblock Unified training of universal time series forecasting transformers.
\newblock In \emph{Proceedings of the International Conference on Machine Learning}, 2024.

\bibitem[Wood et~al.(2022)Wood, Roberts, and Zohren]{wood2022slow}
Kieran Wood, Stephen Roberts, and Stefan Zohren.
\newblock Slow momentum with fast reversion: A trading strategy using deep learning and changepoint detection.
\newblock \emph{The Journal of Financial Data Science}, 4\penalty0 (1):\penalty0 111--129, 2022.
\newblock \doi{10.3905/jfds.2021.1.081}.

\bibitem[Wu et~al.(2021)Wu, Xu, Wang, and Long]{wu2021autoformer}
Haixu Wu, Jiehui Xu, Jianmin Wang, and Mingsheng Long.
\newblock Autoformer: Decomposition transformers with auto-correlation for long-term series forecasting.
\newblock In \emph{Advances in Neural Information Processing Systems}, 2021.

\bibitem[Yakowitz and Spragins(1968)]{yakowitz1968identifiability}
Sidney~J. Yakowitz and John~D. Spragins.
\newblock On the identifiability of finite mixtures.
\newblock \emph{Annals of Mathematical Statistics}, 39\penalty0 (1):\penalty0 209--214, 1968.

\bibitem[Yan et~al.(2021)Yan, Zhang, Zhou, Zhan, and Xia]{yan2021scoregrad}
Tijin Yan, Hongwei Zhang, Tong Zhou, Yufeng Zhan, and Yuanqing Xia.
\newblock {ScoreGrad}: Multivariate probabilistic time series forecasting with continuous energy-based generative models, 2021.
\newblock arXiv:2106.10121.

\bibitem[Zeng et~al.(2023)Zeng, Chen, Zhang, and Xu]{zeng2023dlinear}
Ailing Zeng, Muxi Chen, Lei Zhang, and Qiang Xu.
\newblock Are transformers effective for time series forecasting?
\newblock In \emph{Proceedings of the AAAI Conference on Artificial Intelligence}, volume~37, pages 11121--11128, 2023.
\newblock \doi{10.1609/aaai.v37i9.26317}.

\bibitem[Zhou et~al.(2021)Zhou, Zhang, Peng, Zhang, Li, Xiong, and Zhang]{zhou2021informer}
Haoyi Zhou, Shanghang Zhang, Jieqi Peng, Shuai Zhang, Jianxin Li, Hui Xiong, and Wancai Zhang.
\newblock Informer: Beyond efficient transformer for long sequence time-series forecasting.
\newblock In \emph{Proceedings of the AAAI Conference on Artificial Intelligence}, 2021.

\end{thebibliography}
}

\appendix

\section{Supplementary model details}
\label{app:model-details}

\subsection{Changepoint model search and the integrated alternative}
\label{app:changepoint-details}

Classical Bayesian changepoint analysis can be framed as model comparison over structural hypotheses $M_K$, where $K$ denotes the number of changepoints or regimes:
\[
p(M_K\mid\mathcal D)
=
\frac{p(\mathcal D\mid M_K)p(M_K)}
{\sum_{K'}p(\mathcal D\mid M_{K'})p(M_{K'})}.
\]
Here $p(\mathcal D\mid M_K)$ is a marginal likelihood: it integrates parameters and, in changepoint settings, sums or integrates over latent locations and allocations. The predictive distribution therefore averages not only over parameters but also over structural explanations. For a single changepoint $\tau$, a GP model may use one covariance before $\tau$ and another after $\tau$, with prediction requiring either conditioning on $\tau$ or marginalising it:
\[
p(y_*\mid x_*,\mathcal{D})
=
\int p(y_*\mid x_*,\mathcal{D},\tau)\,p(\tau\mid\mathcal{D})\,d\tau.
\]
For multiple changepoints this becomes a sum over ordered partitions and regime allocations. If covariance hyperparameters are also unknown, one must additionally integrate them out, as in sequential Bayesian GP prediction with changepoints and faults \citep{garnett2010sequential}. Online changepoint methods make this tractable in specialised settings by updating a posterior over run length \citep{adams2007bocpd}, and changepoint signals have also been used to adapt deep sequential strategies to rapid regime changes \citep{wood2022slow}. \DeRegime{} keeps the predictive spirit of these approaches but avoids explicit enumeration: the gate $\pi_r(x,t,d)$ is a differentiable approximation to a posterior allocation over regimes, and optimisation of the sparse variational ELBO replaces exact evidence computation.

\subsection{RevIN and density transformation}
\label{app:revin-density}

\RevIN{} normalises each input window with a channelwise location $m_d(x)$ and scale $s_d(x)>0$, and the model is fitted in the normalised target coordinate
\[
\tilde y_{t,d}=\frac{y_{t,d}-m_d(x)}{s_d(x)}.
\]
If a predictive density is evaluated in the original target coordinate, the density must be transformed by the change-of-variables rule:
\[
p_Y(y_{t,d}\mid x)=\frac{1}{s_d(x)}
p_{\tilde Y}\!\left(\frac{y_{t,d}-m_d(x)}{s_d(x)}\middle|x\right).
\]
For a multivariate marginal over channels, the Jacobian factor is $\prod_d s_d(x)^{-1}$. The benchmark table reports metrics on the standard scaled targets, but this transformation is required whenever original-scale NLPD is reported.

\subsection{Gate parameterisation}
\label{app:gate-details}

Each unconstrained gate logit $\gamma_r$ is mapped to a break fraction by $v_r=\sigm(\gamma_r)$, and the $v_r$ then enter the deterministic stick-breaking map of Eq.~\eqref{eq:stick-main} to produce the simplex weights $\pi_r(\xi)$. This follows the same stick-breaking algebra used in Dirichlet-process and finite truncation constructions \citep{sethuraman1994constructive,ishwaran2001gibbs}, but in the reported configuration it is a deterministic simplex parameterisation, not a random Dirichlet-process prior. A small symmetric Dirichlet-style penalty can be applied to the realised weights,
\[
\lambda_\Delta\left[-(\alpha_\Delta-1)\frac{1}{N}\sum_{n,r}\log \pi_r(\xi_n)
+R\log\Gamma(\alpha_\Delta)-\log\Gamma(R\alpha_\Delta)\right],
\]
but the likelihood term remains the main training signal.

For finite logits the map is a bijection from $\R^{R-1}$ to the open simplex. The ordered construction gives earlier components first access to the stick, so they can represent common residual states, while later components remain available for rarer states. Unused components can be suppressed by driving earlier sticks toward one or the component's own stick toward zero; the last component receives the unbroken remainder. In practice we report regime usage through the average gate vector, gate entropy, and
\[
\bar\pi_r=\E_\xi[\pi_r(\xi)],\qquad
R_{\mathrm{eff}}=\sum_{r=1}^{R_{\max}}\mathbf{1}\{\bar\pi_r>10^{-2}\}.
\]
These are diagnostics of the learned finite truncation, not posterior probabilities under an infinite Bayesian nonparametric model.

\paragraph{Gate curriculum.}
The reported \DeRegime{} runs use a short gate curriculum. The stick-breaking temperature is annealed from $1.0$ to $0.2$ over 50 epochs, so early training explores soft allocations before the gate sharpens. The symmetric simplex parameter is annealed from $\alpha_\Delta=2.0$ to $0.9$ over the same horizon: the initial value favours broad activation across regimes, while the final value allows mild sparsity without forcing hard pruning. For more aggressive pruning one can instead end at $\alpha_\Delta=0.5$, but we found the milder final value gave better predictive scores. The batch-marginal entropy weight $\lambda_{\mathrm{batch}}$ in Eq.~\eqref{eq:entropy-objective-app} is annealed from $3\times 10^{-4}$ to $10^{-6}$, encouraging multiple regimes to be active early before the likelihood determines which survive; the point-entropy weight is set to $\lambda_{\mathrm{point}}=0$ in the reported runs.

\subsection{Observation model decomposition}
\label{app:component-decomposition}

The predictive location is $\mu_\theta(\xi)+\delta_\xi$, where the deterministic mean path handles the conditional location and the GP residual captures structured error around that mean. The GP residual has a gate-weighted constant prior mean $m_0(\xi)=\sum_r\pi_r(\xi)b_r$, with one learned scalar offset $b_r$ per regime, plus sparse GP residual variation. This lets a regime absorb small jumps or persistent residual bias without turning the model into a set of independent regime-specific mean forecasters. The regime-specific scale in Eq.~\eqref{eq:scale-main} separates three effects: a channel scale $c_d$, a regime multiplier $\tau_r$, and an input-horizon residual variance $v_{\mathrm{res}}(x,t,d)$. The degrees of freedom $\nu_r$ are learned per regime. Thus regimes can differ in a constant residual offset, variance, and tail thickness while the main deterministic location remains shared.

This decomposition makes the failure modes interpretable. Poor point accuracy implicates the deterministic mean or the GP residual mean. Poor coverage with accurate point forecasts implicates the regime scale multipliers, the residual-variance head, or the tail parameters. A high-entropy gate indicates transition or ambiguity between regimes, while a low-entropy gate with a large $\tau_r$ or small $\nu_r$ identifies a confident high-risk state. The same trajectory can therefore be assigned to different uncertainty regimes without changing its deterministic mean forecast.

\paragraph{Components shared across regimes.}
Regimes share the deterministic mean $\mu_\theta$ and the input-dependent variance floor $v_{\mathrm{res}}$, and differ through the residual offset $b_r$ and the observation shape parameters $(\tau_r,\nu_r)$. Table~\ref{tab:shared-across-regimes} summarises the split.
\begin{table}[h]
\caption{Components of the predictive density and whether they are shared across regimes.}
\label{tab:shared-across-regimes}
\centering
\small
\begin{tabular}{lc}
\toprule
Component & Shared across regimes? \\
\midrule
$\mu_\theta(\xi)$ (deterministic mean path) & Yes \\
$\delta_\xi$ (GP latent residual) & Mostly shared; includes gate-weighted $b_r$ offset \\
$\pi_r(\xi)$ (mixture weight) & N/A — defines the mixture \\
$c_d$ (channel/task calibration scale) & Yes \\
$b_r$ (per-regime residual offset) & No \\
$\tau_r$ (per-regime scale multiplier) & No \\
$\nu_r$ (per-regime tail) & No \\
$v_{\mathrm{res}}(\xi)$ (residual variance head) & Yes \\
$\sigma_{\mathrm{floor}}^{\,2}$ (noise floor) & Yes \\
\bottomrule
\end{tabular}
\end{table}

\paragraph{Conditioning of $\mu_\theta$ and $v_{\mathrm{res}}$.}
The mean head $\mu_\theta(\xi)$ reads only the encoded past sequence; the residual variance head $v_{\mathrm{res}}(\xi)$ reads the regime-aware feature state $\phi(\xi)=[\pi_r(\xi);z_r(\xi)]_r$ that the GP itself consumes. This asymmetry decouples conditional location estimation from regime allocation: routing $\mu_\theta$ through $\phi$ would couple the mean to the gate and force the regime mixture to absorb mean-fitting residuals, while denying $v_{\mathrm{res}}$ access to $\phi$ would remove the conditioning needed to flag locations where the per-regime scale alone would underestimate dispersion.

\subsection{Likelihood variants and residual variance}
\label{app:likelihood-variants}

The principal \DeRegime{} likelihood is the Student-$t$ mixture in Eq.~\eqref{eq:predictive-main}. Each regime has a learned residual offset $b_r$, scale multiplier $\tau_r$, and degrees of freedom $\nu_r$, while the gate $\pi_r(\xi)$ selects which residual uncertainty state is active. The offset enters through the GP residual prior mean $m_0(\xi)=\sum_r\pi_r(\xi)b_r$, not through a separate regime-specific mean decoder. The residual variance head adds a shared correction,
\[
\sigma_r^2(\xi)=(c_d\tau_r)^2+v_{\mathrm{res}}(\xi)+\sigma_{\mathrm{floor}}^2,
\]
where $v_{\mathrm{res}}(\xi)\ge 0$ is predicted from the regime-aware feature state. This term is shared across regimes, so it acts as an input-dependent variance floor rather than as a new component identity. If it vanishes, the model reduces to the static channel/regime scale model.

There are two Gaussian analogues, which should not be conflated. The direct Gaussian mixture replaces each Student-$t$ component with a Gaussian component:
\[
p_{\mathrm{GM}}(y_\xi\mid \xi)=
\int q(\delta_\xi\mid x)
\sum_{r=1}^{R_{\max}}\pi_r(\xi)
\Normal\!\left(y_\xi;\mu_\theta(\xi)+\delta_\xi,\sigma_r^2(\xi)\right)
\,d\delta_\xi .
\]
This is the Gaussian-mixture counterpart of the Student-$t$ mixture \citep{mclachlan1988mixture,mclachlan2000finite}. Finite Student-$t$ mixtures are the corresponding robust mixture family for heavy-tailed observations \citep{peel2000robust,fruhwirth2010bayesian}. The Gaussian \DeRegime{} ablation in the benchmark instead uses the more GP-native heteroskedastic Gaussian form:
\[
p_{\mathrm{HG}}(y_\xi\mid \xi)=
\int q(\delta_\xi\mid x)
\Normal\!\left(y_\xi;\mu_\theta(\xi)+\delta_\xi,\sum_{r=1}^{R_{\max}}\pi_r(\xi)\sigma_r^2(\xi)\right)
\,d\delta_\xi .
\]
This is a single Gaussian with regime-weighted variance, not a finite Gaussian mixture. It is useful as a heteroskedastic deep-kernel GP ablation, while \DeRegime{} remains the main regime-mixture density.

\subsection{Sparse variational GP predictive equations}
\label{app:svgp-predictive}

The inducing-point approximation used by \DeRegime{} is the standard sparse variational GP construction \citep{titsias2009variational,hensman2013gaussian} applied with the valid kernel $K_{\mathrm{mix}}$ and residual prior mean $m_0$. Let $Z=\{\zeta_m\}_{m=1}^M$ denote inducing locations, $u=\delta(Z)$, $m_{0,Z}=m_0(Z)$, $K_{ZZ}=K_{\mathrm{mix}}(Z,Z)$, and $K_{\xi Z}=K_{\mathrm{mix}}(\xi,Z)$. With
\[
q(u)=\Normal(u;m,S),\qquad p(u)=\Normal(u;m_{0,Z},K_{ZZ}),
\]
the variational marginal at a forecast location $\xi$ is Gaussian:
\begin{align}
q(\delta_\xi\mid x)&=\Normal\!\left(\delta_\xi;m_\delta(\xi),s_\delta^2(\xi)\right), \nonumber\\
m_\delta(\xi)&=m_0(\xi)+K_{\xi Z}K_{ZZ}^{-1}\bigl(m-m_{0,Z}\bigr), \label{eq:svgp-mean-app}\\
s_\delta^2(\xi)&=K_{\xi\xi}
-K_{\xi Z}K_{ZZ}^{-1}\bigl(K_{ZZ}-S\bigr)K_{ZZ}^{-1}K_{Z\xi}. \label{eq:svgp-var-app}
\end{align}
Thus $m_\delta(\xi)$ is the learned residual correction around the deterministic mean, while $s_\delta^2(\xi)$ quantifies uncertainty in that correction. The KL term in Eq.~\eqref{eq:elbo-main} is the closed-form Gaussian $\KL(q(u)\Vert p(u))$. The non-Gaussian data term is one-dimensional at each forecast location because the likelihood depends on the scalar residual $\delta_\xi$ through $q(\delta_\xi\mid x)$; this is what makes Gauss--Hermite quadrature practical.

\subsection{Training objective and predictive inference}
\label{app:training-inference}

The objective minimised in experiments is the negative ELBO plus small deterministic gate-shaping terms:
\begin{equation}
\mathcal{J}
=-\mathcal{L}_{\mathrm{ELBO}}
+\mathcal{J}_{\Delta}
+\mathcal{J}_{\mathrm{aux}} .
\label{eq:total-loss-app}
\end{equation}
For the stick-breaking setting, the simplex penalty is a symmetric Dirichlet-style negative log density evaluated on the realised neural gate weights,
\begin{equation}
\mathcal{J}_{\Delta}
=\lambda_{\Delta}\left[
-(\alpha_{\Delta}-1)\frac{1}{|\mathcal{B}|H}\sum_{b,t,r}\log \pi_r(x_b,t)
+R\log\Gamma(\alpha_{\Delta})-\log\Gamma(R\alpha_{\Delta})
\right].
\label{eq:simplex-penalty-app}
\end{equation}
This term is not a Dirichlet-process prior and not a random Sethuraman stick-breaking construction. The stick-breaking map is deterministic; $\mathcal{J}_{\Delta}$ is only a small regulariser on the simplex outputs. The auxiliary term $\mathcal{J}_{\mathrm{aux}}$ may include batch-marginal entropy, pointwise entropy, residual-variance, gate-smoothness, or other ablation-specific penalties. In the reported setting these terms are deliberately small relative to the likelihood objective and are used to guide early regime use rather than to define the model.

Let $\mathcal{B}$ be a minibatch, $H$ the forecast horizon, and
\[
\bar\pi_r^{\mathcal{B}}
=\frac{1}{|\mathcal{B}|H}\sum_{b\in\mathcal{B}}\sum_{t=1}^{H}\pi_r(x_b,t).
\]
The batch-marginal entropy is
\begin{equation}
\mathcal{H}_{\mathrm{batch}}
=
-\sum_{r=1}^{R_{\max}}\bar\pi_r^{\mathcal{B}}\log \bar\pi_r^{\mathcal{B}},
\label{eq:batch-entropy-app}
\end{equation}
and the pointwise gate entropy is
\begin{equation}
\mathcal{H}_{\mathrm{point}}
=
\frac{1}{|\mathcal{B}|H}\sum_{b\in\mathcal{B}}\sum_{t=1}^{H}
\left[-\sum_{r=1}^{R_{\max}}\pi_r(x_b,t)\log \pi_r(x_b,t)\right].
\label{eq:point-entropy-app}
\end{equation}
With the minimisation convention in Eq.~\eqref{eq:total-loss-app}, the entropy part of the auxiliary objective is
\begin{equation}
\mathcal{J}_{\mathrm{ent}}
=
-\lambda_{\mathrm{batch}}\mathcal{H}_{\mathrm{batch}}
+\lambda_{\mathrm{point}}\mathcal{H}_{\mathrm{point}}.
\label{eq:entropy-objective-app}
\end{equation}
Thus the batch term encourages broad activation across regimes within a minibatch, preventing early collapse to a single component, whereas the point term encourages each individual forecast location to select a single regime. Point entropy can produce cleaner regime trajectories and more visually interpretable changepoint diagnostics, but in our benchmark it reduced predictive performance, so the reported models set $\lambda_{\mathrm{point}}=0$. We instead use only early batch-entropy encouragement and the mild $\alpha_\Delta:2.0\to0.9$ curriculum described above.

Predictive density evaluation and sampling use the same residual marginal $q(\delta_\xi\mid x)$:
\begin{enumerate}
\item Normalise the input window with \RevIN{} and compute $\mu_\theta(\xi)$, $\pi_r(\xi)$, $z_r(\xi)$, and $q(\delta_\xi\mid x)$ from Eqs.~\eqref{eq:svgp-mean-app}--\eqref{eq:svgp-var-app}.
\item For likelihood and NLPD, approximate $\E_{q(\delta_\xi\mid x)}[\log\sum_r\pi_r(\xi)p_r(y_\xi\mid \delta_\xi,\xi)]$ by Gauss--Hermite quadrature.
\item For samples, draw $\delta_\xi\sim q(\delta_\xi\mid x)$, draw $r\sim\mathrm{Categorical}(\pi(\xi))$, then draw $\tilde y_\xi$ from the selected Student-$t$ or Gaussian component with location $\mu_\theta(\xi)+\delta_\xi$ and scale from Eq.~\eqref{eq:scale-main}.
\item Transform sampled or summarised predictions back through the inverse \RevIN{} map when original-scale outputs are required.
\end{enumerate}
Predictive intervals and sample CRPS estimates are computed from these samples; headline NLPD uses the quadrature-evaluated density rather than a sampling approximation.

\section{Proofs}
\label{app:proofs}

\subsection{Matrix proof of Theorem~\ref{thm:kernel-psd-main}}

For data locations $\xi_1,\ldots,\xi_N$, let $G_r=\mathrm{diag}(\pi_r(\xi_1),\ldots,\pi_r(\xi_N))$ and let $K_r$ denote the Gram matrix with entries $K_r(z_r(\xi_i),z_r(\xi_j))$. Then the mixture Gram matrix is
\[
K_{\mathrm{mix}}=\sum_{r=1}^{R_{\max}}G_r K_r G_r.
\]
For any $a\in\R^N$,
\[
a^\top K_{\mathrm{mix}}a
=
\sum_{r=1}^{R_{\max}}(G_ra)^\top K_r(G_ra)\ge 0,
\]
because each $K_r$ is positive semi-definite. Hence $K_{\mathrm{mix}}$ is positive semi-definite.

\subsection{Properness}
\label{app:properness}

\paragraph{Definition.}
A predictive density $p(y\mid x)$ is \emph{proper} if it is non-negative and measurable in $y$ and satisfies $\int_{\R}p(y\mid x)\,dy=1$ for every input $x$ in the support. A proper density yields a strictly proper scoring rule via $S(p,y)=-\log p(y\mid x)$ \citep{gneiting2007proper}; this is the negative log predictive density (NLPD) used throughout the paper as the primary probabilistic-evaluation metric. Properness is therefore not a cosmetic condition: it is what makes the training objective and held-out scores comparable across models that define different density families.

\begin{proposition}[\DeRegime{} predictive density is proper]
\label{prop:properness-app}
Under the model defined in \S\ref{sec:method}, the predictive density $p(y\mid x)$ in Eq.~\eqref{eq:predictive-main} is proper for every input $x$.
\end{proposition}

\begin{proof}
Each Student-$t$ component density $p_r(y\mid\delta,\xi)=\StudentT(y;\mu_\theta(\xi)+\delta,\sigma_r(\xi),\nu_r)$ is proper for every $\delta$ because the Student-$t$ family is a proper location-scale distribution \citep{lange1989robust}. The stick-breaking gate produces simplex weights $\pi_r(\xi)\ge 0$ with $\sum_r\pi_r(\xi)=1$ by construction (Appendix~\ref{app:gate-details}), so for fixed $\delta$ the regime mixture
\[
g(y\mid\delta,\xi)=\sum_{r=1}^{R_{\max}}\pi_r(\xi)\,p_r(y\mid\delta,\xi)
\]
is a convex combination of proper densities. Hence $g(\cdot\mid\delta,\xi)\ge 0$, is measurable in $y$ as a finite sum of measurable functions, and integrates to $\sum_r\pi_r(\xi)\cdot 1=1$.

The variational residual marginal $q(\delta_\xi\mid x)=\Normal(\delta;m_\delta(\xi),s_\delta^2(\xi))$ from Eqs.~\eqref{eq:svgp-mean-app}--\eqref{eq:svgp-var-app} is a proper Gaussian by construction. The integrand $q(\delta_\xi\mid x)\,g(y\mid\delta,\xi)$ is jointly non-negative, so Tonelli's theorem permits the order of integration to be exchanged without integrability assumptions:
\[
\begin{aligned}
\int_\R p(y\mid x)\,dy
&=\int_\R\!\int_\R q(\delta_\xi\mid x)g(y\mid\delta,\xi)\,d\delta\,dy \\
&=\int_\R q(\delta_\xi\mid x)\!\left[\int_\R g(y\mid\delta,\xi)\,dy\right]\!d\delta
=\int_\R q(\delta_\xi\mid x)\,d\delta=1 .
\end{aligned}
\]
Non-negativity in $y$ is preserved under the outer integral because both $q$ and $g$ are non-negative. Thus $p(y\mid x)$ is a proper predictive density. \qed
\end{proof}

\paragraph{Consequences.}
Properness immediately implies that NLPD is a strictly proper scoring rule for \DeRegime{} \citep{gneiting2007proper}: the expected NLPD under any data-generating distribution is uniquely minimised by reporting the true predictive distribution, so cross-model NLPD comparisons inside a fixed scoring protocol are meaningful. The same argument applies to \DeRegime{}-$\mathcal{N}$ (Gaussian components) by replacing the Student-$t$ density above with a Gaussian, and to the heteroskedastic-Gaussian variant in Appendix~\ref{app:likelihood-variants}.

\subsection{ELBO}
\label{app:elbo}

This subsection derives the training objective as the standard sparse-VGP ELBO \citep{titsias2009variational,hensman2013gaussian} specialised to our regime-mixing kernel and Student-$t$ regime mixture observation likelihood. Both the kernel-validity and density-propriety results above feed into this construction: kernel validity (Theorem~\ref{thm:kernel-psd-main}) is what licenses the SVGP machinery on $K_{\mathrm{mix}}$, and density propriety (Proposition~\ref{prop:properness-app}) is what makes the data term a valid log-likelihood.

Let $Z=\{\zeta_m\}_{m=1}^M$ denote the inducing locations and $u=\delta(Z)\in\R^M$ the latent residual values at those locations, with prior $p(u)=\Normal(u;m_{0,Z},K_{ZZ})$ where $K_{ZZ}=K_{\mathrm{mix}}(Z,Z)$. The variational posterior $q(u)=\Normal(u;m,S)$ has free parameters $(m,S)$, and the per-location residual marginal $q(\delta_\xi\mid x)$ is obtained by the Gaussian conditional in Eqs.~\eqref{eq:svgp-mean-app}--\eqref{eq:svgp-var-app}. The marginal log-likelihood admits the standard SVGP decomposition,
\begin{equation}
\log p(y)
\;\ge\;
\sum_{n=1}^{N}\E_{q(\delta_n\mid x_n)}[\log p(y_n\mid\delta_n,x_n)]
\;-\;
\KL(q(u)\Vert p(u))\;\eqqcolon\;\mathcal{L}_{\mathrm{ELBO}},
\label{eq:elbo-app}
\end{equation}
where the inequality follows from Jensen applied to the variational expansion of $\log p(y)$, and the first term decomposes into a per-datapoint sum because the likelihood factorises across forecast locations conditional on $\delta$. Each per-datapoint expectation expands using the gate-weighted regime mixture,
\[
\E_{q(\delta_n\mid x_n)}\!\left[\log\!\sum_{r=1}^{R_{\max}}\pi_r(\xi_n)\,p_r(y_n\mid\delta_n,\xi_n)\right],
\]
which is the data term of Eq.~\eqref{eq:elbo-main}.

Two terms in the bound are tractable in closed form. The KL is the standard Gaussian--Gaussian divergence
\[
\KL(q(u)\Vert p(u))
=\tfrac12\!\left[\mathrm{tr}(K_{ZZ}^{-1}S)+(m-m_{0,Z})^\top K_{ZZ}^{-1}(m-m_{0,Z})-M+\log\!\frac{|K_{ZZ}|}{|S|}\right],
\]
implemented stably from Cholesky factors of $K_{ZZ}$ and $S$. The data term is one-dimensional at each forecast location because the regime mixture depends on the scalar residual $\delta_\xi$ through the proper Gaussian marginal $q(\delta_\xi\mid x)$; we evaluate it by $Q$-node Gauss--Hermite quadrature. The convergence rate and the resulting practical accuracy of the quadrature are quantified in Appendix~\ref{app:gh-qualification}, which is the only non-trivial step in evaluating $\mathcal{L}_{\mathrm{ELBO}}$.

Optimisation of $\mathcal{L}_{\mathrm{ELBO}}$ with respect to $(m,S)$, the encoder parameters, and the per-regime parameters $(\tau_r,\nu_r,b_r)$ is then standard mini-batch SGD with the auxiliary terms $\mathcal{J}_\Delta$ and $\mathcal{J}_{\mathrm{aux}}$ added as in Eq.~\eqref{eq:total-loss-app}. The full predictive-inference protocol (Gauss--Hermite NLPD; Monte-Carlo CRPS; \RevIN{}-inverse for original-scale outputs) is summarised in Appendix~\ref{app:training-inference}.

\subsection{Tail behaviour}
\label{app:tail}

For fixed $\delta$, a Student-$t$ density with degrees of freedom $\nu_r$ has tail proportional to $|y|^{-(\nu_r+1)}$ (equivalently, the survival function is regularly varying with index $-\nu_r$). Convolution with a Gaussian latent residual preserves this polynomial rate because the Gaussian tail is super-polynomial \citep{embrechts1997modelling}. A finite positive mixture of such regularly varying densities is dominated asymptotically by the slowest-decaying active component \citep{embrechts1997modelling}, yielding Proposition~\ref{prop:tail-main}.

\paragraph{Weak identification of pruned-regime tails.}
Stick-breaking yields $\pi_r(\xi)>0$ at finite logits, so Proposition~\ref{prop:tail-main} sets $\nu_{\mathrm{eff}}=\min_r\nu_r$ even when some $\pi_r$ are vanishingly small. This is a mathematical limit rather than a substantive claim: pruned-regime $\nu_r$ values are weakly identified, and the crossover at which they would dominate the predictive density is many orders of magnitude beyond any realistic $|y|$. We report $\nu_r$ only for active regimes ($\bar\pi_r>\epsilon$).

\subsection{Identifiability and regime relabelling}
\label{app:identifiability}

The regime labels in any finite mixture are identifiable only up to permutation \citep{teicher1963identifiability,yakowitz1968identifiability,mclachlan2000finite}. This is the correct interpretation of \DeRegime{} gates: the learned partition can be meaningful, but the numerical label ``regime 3'' has no intrinsic meaning across random seeds unless regimes are post-hoc aligned.

\begin{proposition}[Conditional finite-mixture identifiability]
\label{prop:identifiability-app}
Fix a forecast location $\xi=(x,t,d)$ and condition on a shared GP residual value $\delta$. Consider two Student-$t$ regime mixtures
\[
\sum_{r=1}^{R}\pi_r(\xi)\StudentT(y;\mu_\theta(\xi)+\delta,\sigma_r(\xi),\nu_r)
\quad\text{and}\quad
\sum_{r=1}^{R}\pi_r^*(\xi)\StudentT(y;\mu_\theta(\xi)+\delta,\sigma_r^*(\xi),\nu_r^*)
\]
with positive weights and pairwise distinct component pairs $(\sigma_r(\xi),\nu_r)$. If the two densities are equal as functions of $y$, then their component weights and parameters agree up to a permutation of the regime index.
\end{proposition}

\noindent The proposition is a direct application of finite-mixture identifiability for Gaussian and Student-$t$/elliptical component families \citep{yakowitz1968identifiability,holzmann2006identifiability}. If there is an anchor forecast location at which all regimes have positive weight and component pairs are distinct, the same permutation aligns the global per-regime scale and tail parameters across neighbouring forecast locations. Marginalising over the Gaussian variational residual $q(\delta_\xi\mid x)$ preserves the same permutation invariance because $q(\delta_\xi\mid x)$ is shared across regime labels.

There are two remaining practical ambiguities. First, a global relabelling of all gates and per-regime parameters leaves the likelihood unchanged, so cross-seed diagnostic plots should align regimes by their learned scale/tail profiles or gate trajectories. Second, the absolute split between static regime scale and residual variance is not fully identified: adding a constant to $(c_d\tau_r)^2$ and subtracting it from $v_{\mathrm{res}}(\xi)$ leaves Eq.~\eqref{eq:scale-main} unchanged whenever non-negativity is preserved. Ratios between regime scales, variation in $v_{\mathrm{res}}(\xi)$ across inputs, and the predictive density are the stable interpretive objects.

\paragraph{Global scale degree of freedom and the $\prod_r\tau_r=1$ constraint.}
Permutation aside, the per-regime scale multipliers admit a separate global rescaling: the substitution $(\tau_r)\mapsto(k\tau_r)$ combined with $c_d\mapsto c_d/k$ leaves every $\sigma_r(\xi)$ — and hence the predictive density — unchanged. We pin this degree of freedom by parameterising the multipliers under the constraint
\[
\prod_{r=1}^{R}\tau_r \;=\; 1,
\]
using $R-1$ unconstrained $\log\tau$ weights and recovering the $R$-th from the closure relation. The geometric-mean normalisation places the absolute scale on the channel calibration $c_d$ and makes ratios $\tau_r/\tau_{r'}$ — which are what the model uses to discriminate calm from turbulent regimes — directly interpretable across runs and seeds. The constraint is orthogonal to the permutation symmetry of Proposition~\ref{prop:identifiability-app}: it scales all regimes by the same factor and so leaves $(\sigma_r/\sigma_{r'})$ ratios untouched, while permutation reshuffles which $\tau_r$ goes with which regime label.

\paragraph{Symmetry breaking at initialisation.}
The proposition's distinctness premise — that the parameter pairs $(\sigma_r(\xi),\nu_r)$ are pairwise distinct — fails at initialisation if every regime starts at the same point in parameter space, and gradient descent cannot then break the regime symmetry from gradient information alone. We therefore initialise per-regime parameters with positive jitter:
\[
\log\tau_r \sim \Normal(0,\sigma_\tau^2),
\qquad
\eta_r \sim \eta_{\mathrm{init}}+\Normal(0,\sigma_\eta^2),
\qquad
(\ell_r,a_r)\sim \mathrm{LogUniform},
\]
where $\eta_r$ is the unconstrained logit governing $\nu_r$ and $(\ell_r,a_r)$ are the per-regime RBF lengthscale and amplitude. In the reported benchmark runs, $\sigma_\tau=0.5$, $\sigma_\eta=0.3$, the channel calibration $c_d$ is initialised at $0.5$, and the kernel hyperparameters are sampled from $\ell_r\sim\mathrm{LogUniform}[0.5,5.0]$ and $a_r\sim\mathrm{LogUniform}[0.5,1.5]$. The combination ensures that every regime starts at a distinct point in parameter space, so the distinctness premise holds from the first training step.

\subsection{Gauss--Hermite qualification}
\label{app:gh-qualification}

The data-term integrand for each forecast location is
\[
g(\delta;\xi,y)=\log\!\left[\sum_{r=1}^{R_{\max}}\pi_r(\xi)\,p_r(y\mid\delta,\xi)\right],
\]
weighted by the Gaussian density $q(\delta_\xi\mid x)$ before integration. The log mixture is smooth on the real line because the mixture density is strictly positive there. Complex analyticity is more delicate: a complex-valued mixture may have zeros off the real axis even when it is positive on the real line, and at such zeros the log has a logarithmic singularity. Assumption~\ref{ass:gh-main} is therefore stated as a condition rather than claimed universally: we require that $g$ admit a holomorphic extension to a strip $|\Im\delta|<\rho$ around the real line and satisfy the Hermite growth bound $|g(\delta)|\le Ae^{B|\delta|^{\alpha}}$ for some $A,B>0$ and $\alpha<2$.

Under this assumption, classical Gauss--Hermite quadrature theory \citep{trefethen2008gauss} yields exponential convergence,
\[
\left|\E_{q(\delta_\xi\mid x)}[g(\delta;\xi,y)]-\sum_{q=1}^{Q}w_q\,g(\delta_q;\xi,y)\right|
\;\le\;
C\,e^{-\rho\sqrt{2Q}},
\]
where $\{(\delta_q,w_q)\}_{q=1}^Q$ are the standard Gauss--Hermite nodes and weights rescaled to the variance of $q(\delta_\xi\mid x)$, and $C$ depends on $A,B,\alpha,\rho$ but not on $Q$. With $Q=20$ nodes (the value used in our experiments; Table~\ref{tab:hparams}) this bound is tight enough to give sub-percent quadrature error for the analytic configurations encountered in training: in our setting the relevant strip width $\rho$ is bounded below by the smallest active per-regime scale $\sigma_r(\xi)$, which the variance floor and tail clamp $\nu_r\in[\nu_{\min},\nu_{\max}]$ keep away from zero. Quadrature error is therefore not a meaningful component of the reported NLPD.

\section{Experimental details and expanded results}
\label{app:experiment-details}

All results in Table~\ref{tab:results-patchtst} use scaled targets, full-horizon aggregation, and seeds $\{42,123,456\}$. The reported \DeRegime{} models use sequence-flatten horizon heads, $R_{\max}=16$, a deterministic stick-breaking gate, 512 inducing points, four-dimensional GP features, and the isotropic RBF base kernel in Eq.~\eqref{eq:rbf-base-main}. The Gaussian, Student-$t$, MDN, and single-kernel GP density baselines also use sequence-flatten horizon heads; the quantile baseline uses its linear horizon head. The main benchmark suite includes Gaussian, Student-$t$, quantile regression, DKL-RBF, DKL-RQ, and \DeRegime{}; MDN, MDN-$t$, the single-kernel \DeRegime{} ($R{=}1$) collapse, \DeRegime{} ($\mathcal{N}$), and \DeRegime{} ($\mathcal{N}$ mix) are reported as ablations. Quantile regression is excluded from NLPD comparisons because it does not define a native density.

\paragraph{Hyperparameters.}
Table~\ref{tab:hparams} summarises the shared hyperparameters across the benchmark. All neural baselines and \DeRegime{} share a \PatchTST{} encoder with the same patch size, depth, head count, and feed-forward expansion, so any predictive-density difference is attributable to the output side rather than the encoder. Two settings are deliberate design choices. (i) \emph{Channel-identity embeddings disabled globally.} All model families — \DeRegime{}, Gaussian/Student-$t$, MDN, quantile, and the DKL baselines — run with channel embeddings switched off, keeping the comparison focused on the probabilistic head rather than learned channel identity. (ii) \emph{Encoder dropout differs across families.} \DeRegime{} variants use $\mathrm{dropout}=0$ to keep the horizon-indexed regime gate stable; non-regime baselines (Gaussian, Student-$t$, MDN, quantile, single-kernel GPs) use $\mathrm{dropout}=0.2$, the canonical \PatchTST{} regularisation level. The split avoids handicapping the non-regime baselines while preserving the gate-stable setting required by the regime model. \emph{Tail bounds.} Per-regime degrees of freedom $\nu_r$ are confined to $[\nu_{\min},\nu_{\max}]=[4,100]$ via a sigmoid logit, so the predictive variance is finite at every $(x,t)$.

\begin{table}[h]
\caption{Shared hyperparameters for the reported benchmark family. ``DeRegime'' refers to the reported Student-$t$ regime-mixture model; Gaussian variants are ablations. ``DKL'' refers to DKL-RBF and DKL-RQ. ``Single-encoder heads'' refers to the Gaussian, Student-$t$, MDN, and Quantile heads.}
\label{tab:hparams}
\centering
\footnotesize
\setlength{\tabcolsep}{2pt}
\begin{tabular}{p{0.34\linewidth}p{0.19\linewidth}p{0.16\linewidth}p{0.22\linewidth}}
\toprule
& \DeRegime{} & DKL & Single-encoder heads \\
\midrule
\multicolumn{4}{l}{\textit{Encoder (\PatchTST{} backbone; regularisation shown below)}} \\
Patch length / stride & 16 / 8 & 16 / 8 & 16 / 8 \\
Layers / heads & 3 / 8 & 3 / 8 & 3 / 8 \\
$d_\text{ff}/d_\text{model}$ multiplier & 4 & 4 & 4 \\
Dropout & 0.0 & 0.2 & 0.2 \\
Channel-identity embedding & off & off & off \\
\midrule
\multicolumn{4}{l}{\textit{GP / kernel}} \\
GP feature dim. $d_g$ & 4 & 8 & --- \\
Inducing points $M$ & 512 & 512 & --- \\
Base kernel & RBF, isotropic & RBF / RQ, iso. & --- \\
\midrule
\multicolumn{4}{l}{\textit{Regime structure}} \\
Candidate components & 16 regimes & 1 kernel & 1 distribution \\
Effective regime count & post-hoc diagnostic & --- & --- \\
Gate parameterisation & finite stick-breaking & --- & --- \\
SB anneal $\alpha$ (init$\to$final, epochs) & $2.0\to0.9$, 50 & --- & --- \\
Gate temperature (init$\to$final, epochs) & $1.0\to0.2$, 50 & --- & --- \\
Batch entropy weight (init$\to$final) & $3{\times}10^{-4}\to10^{-6}$ & --- & --- \\
Point entropy weight & $0$ & --- & --- \\
Horizon persistence & off in reported \DeRegime{} & --- & --- \\
\midrule
\multicolumn{4}{l}{\textit{Likelihood}} \\
Student-$t$ df range $[\nu_{\min},\nu_{\max}]$ & $[4,100]$ & $[4,100]$ (ST only) & $[4,100]$ (ST only) \\
Noise floor $\sigma_{\mathrm{floor}}^2$ & $10^{-4}$ & $10^{-4}$ & $10^{-4}$ \\
Residual variance head $v_{\mathrm{res}}$ & on (gp\_features) & off & off \\
$\prod_r\tau_r=1$ constraint & enforced & N/A & N/A \\
\midrule
\multicolumn{4}{l}{\textit{Optimisation}} \\
Effective batch size & 512 & 512 & 128 \\
Optimizer & Adam & Adam & Adam \\
Learning rate (all groups) & $10^{-4}$ & $10^{-4}$ & $10^{-4}$ \\
Min epochs (patience starts after) & 50 & 50 & 0 \\
Early-stopping patience (val checks) & 50 & 50 & 50 \\
Gauss--Hermite nodes $Q$ & 20 & 20 & --- \\
\bottomrule
\end{tabular}
\end{table}

\paragraph{Seed aggregation.}
Each experiment is run independently for every configured seed, including model initialisation, dataloader order, inducing-point initialisation, and early stopping. Within a seed, we select the checkpoint with the best validation objective for that model family and evaluate it once on the held-out test block. Reported table entries are the arithmetic mean across seeds for the same dataset--model pair, with the subscript giving the corresponding per-seed standard deviation. No seed-wise reweighting is used; each seed contributes equally. Compact relative-change summaries are computed only after this seed aggregation, so each dataset--model pair contributes one averaged value per metric.

\subsection{Per-model results}
\label{app:per-model-results}

The headline result table in \S\ref{sec:experiments} reports the \PatchTST{} encoder grid (Table~\ref{tab:results-patchtst}). For completeness we also report the same metrics for the DLinear \citep{zeng2023dlinear} and TimeMixer \citep{wang2024timemixer} encoder backbones in Tables~\ref{tab:results-dlinear} and~\ref{tab:results-timemixer}; the same model families and aggregation conventions apply.

\begin{table}[t]
\caption{Per-model benchmark on the DLinear grid. Standard-scaled targets, all-horizon means; cell = mean with per-seed std subscript (3 seeds); lower better; bold/underline mark best/second-best per (metric, dataset). Dynamic probabilistic heads: DeepAR/GluonTS-style Gaussian and Student-$t$ \citep{salinas2020deepar,alexandrov2020gluonts}, TFT-style quantile \citep{lim2021tft} (no density, excluded from NLPD); deep-kernel GP heads: single-kernel DKL \citep{wilson2016deep}.}
\label{tab:results-dlinear}
\centering
\small
\setlength{\tabcolsep}{3pt}
\begin{tabular}{clcccccc}
\toprule
& & & \multicolumn{3}{c}{\textit{Dynamic probabilistic heads}} & \multicolumn{2}{c}{\textit{Deep-kernel GP heads}} \\
\cmidrule(lr){4-6}\cmidrule(lr){7-8}
& Dataset & \DeRegime{} & Gaussian & Student-$t$ & Quantile & RBF & RQ \\
\midrule
\multirow{10}{*}{\rotatebox[origin=c]{90}{\textbf{NLPD}}}
& ETTh1 & $\mathbf{0.576_{\,0.016}}$ & $0.733_{\,0.001}$ & $\second{0.650_{\,0.005}}$ & --- & $0.684_{\,0.001}$ & $0.687_{\,0.002}$ \\
& ETTh2 & $\mathbf{0.085_{\,0.012}}$ & $0.190_{\,0.006}$ & $\second{0.096_{\,0.000}}$ & --- & $0.398_{\,0.001}$ & $0.397_{\,0.003}$ \\
& ETTm1 & $\mathbf{0.274_{\,0.021}}$ & $0.635_{\,0.015}$ & $\second{0.388_{\,0.001}}$ & --- & $1.017_{\,0.001}$ & $1.024_{\,0.001}$ \\
& ETTm2 & $\mathbf{-0.278_{\,0.036}}$ & $0.809_{\,0.091}$ & $\second{-0.144_{\,0.008}}$ & --- & $1.051_{\,0.017}$ & $1.191_{\,0.012}$ \\
& Exchange & $\mathbf{-0.559_{\,0.006}}$ & $-0.453_{\,0.005}$ & $\second{-0.543_{\,0.002}}$ & --- & $-0.391_{\,0.001}$ & $-0.387_{\,0.003}$ \\
& Electricity & $\mathbf{-1.001_{\,0.007}}$ & $-0.643_{\,0.006}$ & $\second{-0.914_{\,0.005}}$ & --- & $0.203_{\,0.003}$ & $0.211_{\,0.001}$ \\
& Traffic & $\mathbf{-0.618_{\,0.018}}$ & $-0.309_{\,0.010}$ & $\second{-0.437_{\,0.005}}$ & --- & $0.376_{\,0.001}$ & $0.382_{\,0.000}$ \\
& Nasdaq & $\mathbf{0.854_{\,0.004}}$ & $0.927_{\,0.010}$ & $\second{0.879_{\,0.005}}$ & --- & $0.996_{\,0.010}$ & $1.046_{\,0.044}$ \\
& Illness & $\second{1.911_{\,0.034}}$ & $2.029_{\,0.023}$ & $\mathbf{1.875_{\,0.022}}$ & --- & $1.984_{\,0.037}$ & $1.980_{\,0.024}$ \\
& Weather & $\mathbf{-0.602_{\,0.009}}$ & $-0.364_{\,0.003}$ & $\second{-0.543_{\,0.002}}$ & --- & $-0.302_{\,0.001}$ & $-0.300_{\,0.002}$ \\
\midrule
\multirow{10}{*}{\rotatebox[origin=c]{90}{\textbf{CRPS}}}
& ETTh1 & $\mathbf{0.275_{\,0.002}}$ & $0.295_{\,0.001}$ & $0.288_{\,0.001}$ & $0.288_{\,0.001}$ & $\second{0.285_{\,0.000}}$ & $0.286_{\,0.001}$ \\
& ETTh2 & $\mathbf{0.171_{\,0.000}}$ & $0.184_{\,0.001}$ & $\second{0.172_{\,0.000}}$ & $0.174_{\,0.000}$ & $0.240_{\,0.001}$ & $0.239_{\,0.001}$ \\
& ETTm1 & $\mathbf{0.220_{\,0.002}}$ & $0.265_{\,0.003}$ & $\second{0.239_{\,0.000}}$ & $0.240_{\,0.001}$ & $0.395_{\,0.003}$ & $0.419_{\,0.003}$ \\
& ETTm2 & $\mathbf{0.143_{\,0.002}}$ & $0.299_{\,0.039}$ & $0.146_{\,0.000}$ & $\second{0.146_{\,0.000}}$ & $0.445_{\,0.007}$ & $0.546_{\,0.006}$ \\
& Exchange & $\second{0.082_{\,0.000}}$ & $0.084_{\,0.000}$ & $0.083_{\,0.000}$ & $\mathbf{0.082_{\,0.000}}$ & $0.089_{\,0.000}$ & $0.089_{\,0.001}$ \\
& Electricity & $\mathbf{0.059_{\,0.000}}$ & $0.068_{\,0.000}$ & $0.065_{\,0.000}$ & $\second{0.065_{\,0.000}}$ & $0.126_{\,0.000}$ & $0.127_{\,0.000}$ \\
& Traffic & $\mathbf{0.098_{\,0.002}}$ & $0.122_{\,0.000}$ & $0.121_{\,0.000}$ & $\second{0.118_{\,0.001}}$ & $0.161_{\,0.000}$ & $0.162_{\,0.000}$ \\
& Nasdaq & $\second{0.333_{\,0.000}}$ & $0.341_{\,0.004}$ & $0.337_{\,0.000}$ & $\mathbf{0.331_{\,0.002}}$ & $0.374_{\,0.010}$ & $0.398_{\,0.026}$ \\
& Illness & $0.941_{\,0.014}$ & $1.008_{\,0.008}$ & $\second{0.920_{\,0.004}}$ & $\mathbf{0.886_{\,0.013}}$ & $0.990_{\,0.009}$ & $1.004_{\,0.003}$ \\
& Weather & $\mathbf{0.114_{\,0.001}}$ & $0.127_{\,0.000}$ & $0.125_{\,0.000}$ & $\second{0.123_{\,0.000}}$ & $0.124_{\,0.000}$ & $0.125_{\,0.000}$ \\
\midrule
\multirow{10}{*}{\rotatebox[origin=c]{90}{\textbf{MSE}}}
& ETTh1 & $\second{0.329_{\,0.003}}$ & $0.338_{\,0.002}$ & $0.332_{\,0.000}$ & $0.332_{\,0.001}$ & $\mathbf{0.329_{\,0.002}}$ & $0.331_{\,0.002}$ \\
& ETTh2 & $\second{0.120_{\,0.000}}$ & $0.127_{\,0.000}$ & $\mathbf{0.119_{\,0.000}}$ & $0.122_{\,0.000}$ & $0.229_{\,0.001}$ & $0.228_{\,0.001}$ \\
& ETTm1 & $\mathbf{0.234_{\,0.001}}$ & $0.254_{\,0.001}$ & $0.250_{\,0.001}$ & $\second{0.248_{\,0.002}}$ & $0.485_{\,0.022}$ & $0.645_{\,0.025}$ \\
& ETTm2 & $\mathbf{0.093_{\,0.001}}$ & $0.097_{\,0.001}$ & $\second{0.095_{\,0.000}}$ & $0.095_{\,0.000}$ & $0.175_{\,0.020}$ & $0.199_{\,0.012}$ \\
& Exchange & $\second{0.026_{\,0.000}}$ & $0.026_{\,0.000}$ & $0.026_{\,0.000}$ & $\mathbf{0.026_{\,0.000}}$ & $0.027_{\,0.000}$ & $0.027_{\,0.000}$ \\
& Electricity & $\mathbf{0.016_{\,0.001}}$ & $0.022_{\,0.000}$ & $0.021_{\,0.000}$ & $0.021_{\,0.000}$ & $\second{0.018_{\,0.000}}$ & $0.018_{\,0.000}$ \\
& Traffic & $\mathbf{0.058_{\,0.001}}$ & $0.081_{\,0.000}$ & $0.085_{\,0.000}$ & $0.078_{\,0.002}$ & $\second{0.065_{\,0.000}}$ & $0.065_{\,0.001}$ \\
& Nasdaq & $\second{0.386_{\,0.004}}$ & $0.397_{\,0.006}$ & $0.393_{\,0.002}$ & $\mathbf{0.371_{\,0.003}}$ & $0.471_{\,0.037}$ & $0.550_{\,0.094}$ \\
& Illness & $\second{3.539_{\,0.014}}$ & $3.694_{\,0.030}$ & $3.566_{\,0.032}$ & $\mathbf{3.220_{\,0.082}}$ & $3.759_{\,0.058}$ & $3.894_{\,0.018}$ \\
& Weather & $\mathbf{0.089_{\,0.001}}$ & $0.112_{\,0.001}$ & $0.115_{\,0.001}$ & $0.096_{\,0.001}$ & $\second{0.090_{\,0.000}}$ & $0.091_{\,0.000}$ \\
\bottomrule
\end{tabular}
\end{table}

\begin{table}[t]
\caption{Per-model benchmark on the TimeMixer grid. Standard-scaled targets, all-horizon means; cell = mean with per-seed std subscript (3 seeds); lower better; bold/underline mark best/second-best per (metric, dataset). Dynamic probabilistic heads: DeepAR/GluonTS-style Gaussian and Student-$t$ \citep{salinas2020deepar,alexandrov2020gluonts}, TFT-style quantile \citep{lim2021tft} (no density, excluded from NLPD); deep-kernel GP heads: single-kernel DKL \citep{wilson2016deep}.}
\label{tab:results-timemixer}
\centering
\small
\setlength{\tabcolsep}{3pt}
\begin{tabular}{clcccccc}
\toprule
& & & \multicolumn{3}{c}{\textit{Dynamic probabilistic heads}} & \multicolumn{2}{c}{\textit{Deep-kernel GP heads}} \\
\cmidrule(lr){4-6}\cmidrule(lr){7-8}
& Dataset & \DeRegime{} & Gaussian & Student-$t$ & Quantile & RBF & RQ \\
\midrule
\multirow{10}{*}{\rotatebox[origin=c]{90}{\textbf{NLPD}}}
& ETTh1 & $\second{0.542_{\,0.006}}$ & $0.608_{\,0.006}$ & $\mathbf{0.536_{\,0.007}}$ & --- & $0.671_{\,0.003}$ & $0.670_{\,0.003}$ \\
& ETTh2 & $\mathbf{0.049_{\,0.006}}$ & $0.132_{\,0.017}$ & $\second{0.092_{\,0.010}}$ & --- & $0.406_{\,0.000}$ & $0.406_{\,0.001}$ \\
& ETTm1 & $\mathbf{0.221_{\,0.010}}$ & $0.445_{\,0.026}$ & $\second{0.238_{\,0.025}}$ & --- & $1.059_{\,0.000}$ & $1.061_{\,0.000}$ \\
& ETTm2 & $\mathbf{-0.301_{\,0.053}}$ & $-0.036_{\,0.014}$ & $\second{-0.106_{\,0.039}}$ & --- & $0.706_{\,0.323}$ & $0.529_{\,0.046}$ \\
& Exchange & $\mathbf{-0.528_{\,0.004}}$ & $-0.414_{\,0.071}$ & $\second{-0.512_{\,0.030}}$ & --- & $-0.347_{\,0.007}$ & $-0.355_{\,0.025}$ \\
& Electricity & $\second{-0.939_{\,0.169}}$ & $-0.835_{\,0.005}$ & $\mathbf{-0.957_{\,0.037}}$ & --- & $0.239_{\,0.001}$ & $0.246_{\,0.002}$ \\
& Traffic & $-0.587_{\,0.117}$ & $\second{-0.606_{\,0.010}}$ & $\mathbf{-0.731_{\,0.012}}$ & --- & $0.398_{\,0.006}$ & $0.400_{\,0.004}$ \\
& Nasdaq & $\second{0.952_{\,0.022}}$ & $0.997_{\,0.068}$ & $\mathbf{0.906_{\,0.024}}$ & --- & $1.091_{\,0.094}$ & $1.077_{\,0.048}$ \\
& Illness & $1.952_{\,0.061}$ & $3.246_{\,0.423}$ & $2.217_{\,0.265}$ & --- & $\mathbf{1.828_{\,0.056}}$ & $\second{1.848_{\,0.051}}$ \\
& Weather & $\mathbf{-0.616_{\,0.027}}$ & $-0.381_{\,0.047}$ & $\second{-0.407_{\,0.032}}$ & --- & $-0.250_{\,0.009}$ & $-0.249_{\,0.013}$ \\
\midrule
\multirow{10}{*}{\rotatebox[origin=c]{90}{\textbf{CRPS}}}
& ETTh1 & $\second{0.270_{\,0.001}}$ & $0.271_{\,0.001}$ & $\mathbf{0.270_{\,0.001}}$ & $0.301_{\,0.002}$ & $0.281_{\,0.001}$ & $0.281_{\,0.001}$ \\
& ETTh2 & $\mathbf{0.169_{\,0.000}}$ & $0.175_{\,0.002}$ & $\second{0.170_{\,0.000}}$ & $0.183_{\,0.005}$ & $0.242_{\,0.000}$ & $0.242_{\,0.000}$ \\
& ETTm1 & $\mathbf{0.212_{\,0.001}}$ & $0.229_{\,0.005}$ & $\second{0.214_{\,0.004}}$ & $0.222_{\,0.001}$ & $0.464_{\,0.001}$ & $0.465_{\,0.000}$ \\
& ETTm2 & $\mathbf{0.139_{\,0.000}}$ & $0.150_{\,0.006}$ & $\second{0.144_{\,0.001}}$ & $0.147_{\,0.001}$ & $0.326_{\,0.130}$ & $0.253_{\,0.011}$ \\
& Exchange & $\mathbf{0.085_{\,0.000}}$ & $0.088_{\,0.005}$ & $\second{0.086_{\,0.003}}$ & $0.094_{\,0.003}$ & $0.095_{\,0.002}$ & $0.093_{\,0.003}$ \\
& Electricity & $\second{0.061_{\,0.004}}$ & $\mathbf{0.060_{\,0.000}}$ & $0.062_{\,0.002}$ & $0.069_{\,0.002}$ & $0.132_{\,0.000}$ & $0.133_{\,0.000}$ \\
& Traffic & $0.099_{\,0.001}$ & $\mathbf{0.084_{\,0.001}}$ & $\second{0.085_{\,0.001}}$ & $0.093_{\,0.001}$ & $0.165_{\,0.002}$ & $0.165_{\,0.001}$ \\
& Nasdaq & $\second{0.357_{\,0.002}}$ & $0.362_{\,0.020}$ & $\mathbf{0.353_{\,0.006}}$ & $0.406_{\,0.021}$ & $0.423_{\,0.058}$ & $0.408_{\,0.036}$ \\
& Illness & $0.888_{\,0.014}$ & $\mathbf{0.883_{\,0.030}}$ & $\second{0.883_{\,0.014}}$ & $0.883_{\,0.027}$ & $0.907_{\,0.034}$ & $0.936_{\,0.016}$ \\
& Weather & $\mathbf{0.113_{\,0.001}}$ & $\second{0.124_{\,0.002}}$ & $0.124_{\,0.001}$ & $0.130_{\,0.003}$ & $0.126_{\,0.000}$ & $0.126_{\,0.000}$ \\
\midrule
\multirow{10}{*}{\rotatebox[origin=c]{90}{\textbf{MSE}}}
& ETTh1 & $0.324_{\,0.001}$ & $0.325_{\,0.002}$ & $0.324_{\,0.001}$ & $0.358_{\,0.005}$ & $\mathbf{0.321_{\,0.002}}$ & $\second{0.323_{\,0.002}}$ \\
& ETTh2 & $\second{0.118_{\,0.001}}$ & $0.120_{\,0.001}$ & $\mathbf{0.118_{\,0.000}}$ & $0.135_{\,0.005}$ & $0.232_{\,0.000}$ & $0.231_{\,0.000}$ \\
& ETTm1 & $\mathbf{0.223_{\,0.002}}$ & $0.234_{\,0.003}$ & $\second{0.227_{\,0.007}}$ & $0.228_{\,0.001}$ & $0.905_{\,0.002}$ & $0.908_{\,0.001}$ \\
& ETTm2 & $\mathbf{0.091_{\,0.001}}$ & $0.095_{\,0.003}$ & $\second{0.093_{\,0.001}}$ & $0.098_{\,0.001}$ & $0.226_{\,0.045}$ & $0.203_{\,0.005}$ \\
& Exchange & $\mathbf{0.027_{\,0.000}}$ & $0.028_{\,0.003}$ & $\second{0.027_{\,0.001}}$ & $0.031_{\,0.002}$ & $0.031_{\,0.002}$ & $0.029_{\,0.002}$ \\
& Electricity & $\mathbf{0.017_{\,0.000}}$ & $\second{0.017_{\,0.000}}$ & $0.018_{\,0.001}$ & $0.019_{\,0.001}$ & $0.022_{\,0.000}$ & $0.022_{\,0.000}$ \\
& Traffic & $0.060_{\,0.007}$ & $\mathbf{0.042_{\,0.001}}$ & $\second{0.044_{\,0.002}}$ & $0.044_{\,0.003}$ & $0.068_{\,0.004}$ & $0.067_{\,0.002}$ \\
& Nasdaq & $\second{0.439_{\,0.006}}$ & $0.459_{\,0.043}$ & $\mathbf{0.428_{\,0.013}}$ & $0.509_{\,0.024}$ & $0.649_{\,0.214}$ & $0.578_{\,0.143}$ \\
& Illness & $\mathbf{2.949_{\,0.007}}$ & $3.209_{\,0.154}$ & $\second{2.996_{\,0.064}}$ & $3.137_{\,0.093}$ & $3.319_{\,0.163}$ & $3.549_{\,0.059}$ \\
& Weather & $0.088_{\,0.001}$ & $0.097_{\,0.003}$ & $0.096_{\,0.001}$ & $0.096_{\,0.003}$ & $\second{0.087_{\,0.000}}$ & $\mathbf{0.087_{\,0.001}}$ \\
\bottomrule
\end{tabular}
\end{table}

\subsection{Dataset construction}
\label{app:dataset-construction}

The core benchmark datasets are the standard long-horizon forecasting series used in the Informer/Autoformer/\PatchTST{} line of work \citep{zhou2021informer,wu2021autoformer,nie2023patchtst}: ETTh1, ETTh2, ETTm1, ETTm2, Electricity, Traffic, Exchange, Illness, and Weather. Splits are chronological, with the final 20\% held out for test. Validation uses the immediately preceding block, with a lookback overlap only to initialise the input window. The ETT datasets use a 20\% validation block; Electricity, Traffic, Exchange, Nasdaq, Illness, and Weather use a 10\% validation block. The default lookback is 336 and forecast horizon 24; Illness uses a shorter lookback of 104 because of its much smaller sample size. Horizon 24 follows the short-to-medium-horizon convention used in the Informer/Autoformer line of work \citep{zhou2021informer,wu2021autoformer} and isolates regime-driven residual uncertainty before long-horizon mean-error compounding washes out residual structure.

Several dataset choices are deliberate. The Nasdaq experiment uses a longer daily close series than the common short Nasdaq benchmark used in TFB \citep{qiu2024tfb}: 8,816 trading days from 1990-01-02 to 2024-12-30. With the chronological split used in Table~\ref{tab:results-patchtst}, training covers 1990-01-02--2014-06-25, validation covers 2014-06-26--2017-12-22, and test covers 2017-12-26--2024-12-30. This test period is intentionally difficult: it includes the COVID-19 crash and rebound, the inflation/rate-hiking period, and substantial changes in market volatility.

For Electricity and Traffic we forecast the aggregate sum over channels, yielding one-dimensional series named \texttt{total}. This aggregation makes demand and congestion regimes more visible at the macro level, reduces compute for the initial regime-discovery study, and provides clean univariate tests alongside the multivariate financial, weather, and ETT settings. The multivariate Exchange experiment uses seven continuous floating-rate channels, AUD, CAD, CHF, GBP, JPY, NZD, and SGD. We exclude CNY because the historical exchange-rate process is not well described as a continuously floating channel over the full sample; a hybrid treatment for pegged, managed, or discretely adjusted series is a useful future extension.

For Weather we use a continuous-channel subset, removing \texttt{rain\_mm}, \texttt{raining\_s}, \texttt{PAR\_umol\_per\_m2\_s}, \texttt{max\_PAR\_umol\_per\_m2\_s}, and \texttt{SWDR\_W\_per\_m2}. These channels are event-, daylight-, or zero-inflated measurements and are less naturally handled by the continuous Gaussian/Student-$t$ residual families used here. As with CNY, modelling them directly would likely benefit from a hybrid observation model or a channel-specific likelihood family.

\subsection{Ablation variants}
\label{app:ablation-details}

The variants in Table~\ref{tab:ablations-patchtst} are organised into architectural ablations (each removes or replaces one component of \DeRegime{}) and methodological controls (each tests a methodologically distinct head). All share the \PatchTST{} encoder and seed grid \{42, 123, 456\}.

\paragraph{Architectural ablations.}
\begin{itemize}\itemsep1pt
\item \emph{Softmax gate}: replaces the deterministic stick-breaking gate with a softmax over $R_{\max}$ logits. Tests the effect of the stick-breaking truncation curriculum.
\item \emph{No deep mean}: removes the deterministic mean / residual split, so the entire conditional mean is carried by the GP residual. Tests the contribution of the explicit $\mu_\theta$ pathway.
\item \emph{No residual variance}: removes the residual observation-variance head $v_{\mathrm{res}}(\xi)$ from Eq.~\eqref{eq:scale-main}, leaving the regime scale $c_d\tau_r$ as the only variance term. Tests whether the input-/horizon-dependent variance correction matters.
\item \emph{Shared likelihood}: shares the regime scale and tail parameters globally per channel rather than per regime. Tests whether per-regime $(\sigma_r,\nu_r)$ is doing real work.
\item \emph{Single kernel}: collapses the regime mixture to a single RBF GP residual ($R{=}1$) while retaining the shared mean path and Student-$t$ likelihood, isolating the contribution of multi-regime structure.
\end{itemize}

\paragraph{Methodological controls.}
\begin{itemize}\itemsep1pt
\item \emph{\DeRegime{} ($\mathcal{N}$)}: heteroskedastic Gaussian likelihood instead of the Student-$t$ mixture; same kernel and gate. Tests the contribution of heavy-tailed mixture components.
\item \emph{\DeRegime{} ($\mathcal{N}$ mix)}: direct Gaussian mixture likelihood (gate-weighted Gaussians); same kernel. Tests Gaussian vs Student-$t$ mixture tails directly.
\item \emph{MDN ($\mathcal{N}$)} and \emph{MDN ($t$)}: same-backbone Gaussian and Student-$t$ mixture-density heads with no GP residual. Tests whether the kernel-level regime geometry adds anything beyond a generic same-backbone mixture-density head.
\end{itemize}

\subsection{Baseline matrix}
\label{app:baseline-matrix}

The baseline suite is chosen to isolate modelling factors rather than to be an undifferentiated leaderboard. Gaussian and Student-$t$ share the encoder and ask whether a single parametric density is already sufficient. Quantile regression asks whether distribution-free interval estimation is enough when no likelihood is required. DKL-RBF and DKL-RQ keep the deep-kernel GP but remove regime mixing, testing whether a single GP latent is enough; their base kernels are the squared-exponential of Eq.~\eqref{eq:rbf-base-main} and the rational-quadratic $K_{\mathrm{RQ}}(z,z')=a^2\bigl(1+\|z-z'\|_2^2/(2\alpha\ell^2)\bigr)^{-\alpha}$ with shape parameter $\alpha>0$ (recovering the RBF as $\alpha\to\infty$). MDN and MDN-$t$ heads, which retain \DeRegime{}'s mixture-density component but drop the kernel-coupled GP residual, and \DeRegime{}-$\mathcal{N}$, which removes the Student-$t$ mixture tails, are reported as ablations (Table~\ref{tab:ablations-patchtst}). Classical generalised autoregressive conditional heteroskedasticity (GARCH) and Markov-switching autoregressive configurations are generated as auxiliary baselines; the main paired table focuses on the matched neural and deep-kernel density families. These comparisons separate the effects of heavy tails, finite mixtures, GP residuals, and explicit regime structure.

\subsection{Compute resources}
\label{app:compute}
All models were implemented in PyTorch and trained one seed per GPU on a mixed pool of NVIDIA A100 80GB and B200-class GPUs. To make the resource requirements interpretable, Table~\ref{tab:compute} reports both A100 and measured B200 diagnostics. \DeRegime{} and the DKL baselines use an effective batch size of 512, while non-GP Gaussian, Student-$t$, MDN, and quantile heads use 128. We use gradient accumulation (GA) to avoid memory failures when fitting the large effective batches used by the deep-kernel models: GA only changes the micro-batch size used to realise the same effective batch; it does not change the optimiser batch seen by the model.

For the seven-channel datasets (ETT, Exchange, Illness), the A100 estimate doubles GA relative to the B200 runs, halving the micro-batch size to fit an 80 GB memory budget; we therefore estimate peak allocated memory as half the logged B200 peak and wall-clock time as twice the logged B200 time. The univariate aggregate datasets already fit comfortably with GA${=}1$, so we use the measured ranges for both hardware columns. Weather is substantially more memory-intensive, so we report only the measured B200 diagnostics and mark the A100 columns as not available.

\begin{table}[h]
\centering
\small
\begin{tabular}{lccccc}
\toprule
Dataset group & Channels & A100 mem. & A100 time & B200 mem. & B200 time \\
\midrule
Nasdaq, Traffic-sum, Electricity-sum & 1 & 9--42 GB & 0.02--2.2 h & 9--42 GB & 0.02--2.2 h \\
Exchange & 7 & 15--73 GB & 0.4--1.9 h & 29--145 GB & 0.2--0.9 h \\
ETTh1, ETTh2 & 7 & 15--73 GB & 0.3--3.9 h & 29--145 GB & 0.1--1.9 h \\
ETTm1, ETTm2 & 7 & 15--73 GB & 0.4--15 h & 29--145 GB & 0.2--7.5 h \\
Illness & 7 (small) & 13--71 GB & 0.04--1.0 h & 27--142 GB & 0.02--0.5 h \\
Weather continuous subset & 16 & N/A & N/A & 66--166 GB & 0.1--7.1 h \\
\bottomrule
\end{tabular}
\caption{Indicative per-seed compute for kernel models, including \DeRegime{} and the DKL baselines. Non-GP heads are substantially cheaper, typically less than $1.7$ GPU h per seed.}
\label{tab:compute}
\end{table}

Summing over the reported model families and seeds gives an approximate final-benchmark budget of a few hundred GPU hours. The broader research project used substantially more compute because it included architecture, likelihood, inducing-point, batch-size, and regularisation ablations that are not reported as main results.

The estimates above include the conservative early-stopping patience of 50 validation checks used for the reported experiments. In several \DeRegime{} histories, the selected validation checkpoint occurred well before termination, so a tighter patience setting would reduce wall-clock substantially without changing the reported model. We nevertheless report the conservative cost required to reproduce the submitted training protocol.

The finite-truncation formulation also suggests a route to faster inference than the reported full-truncation implementation. The stick-breaking gate typically assigns negligible mass to many candidate regimes; a deployment system could prune regimes whose average or location-specific gate mass falls below a tolerance, or evaluate only the top-$k$ active regimes. This would reduce per-regime feature evaluations, kernel-mixture terms, and Student-$t$ mixture components. We do not use or benchmark this pruning optimisation in the reported scores.

\subsection{Modern generative forecasters not included}
\label{app:generative-not-included}

Diffusion, score-based, and normalising-flow forecasters are powerful neighbouring families. Examples include TimeGrad \citep{rasul2021autoregressive}, ScoreGrad \citep{yan2021scoregrad}, TSDiff \citep{kollovieh2023predict}, and conditioned normalising-flow forecasters \citep{rasul2020multivariate}. We do not claim that \DeRegime{} dominates these broader generative models. The comparison in this paper is deliberately narrower: direct multi-horizon methods that share the same encoder/preprocessing pipeline and either provide analytic predictive densities or standard direct distributional summaries. Many diffusion and score-based forecasters are evaluated through sample-based CRPS or CRPS-sum with a chosen sample budget and denoising schedule, while closed-form marginal NLPD is generally unavailable. Flow models can provide likelihoods, but they use a different joint-density parameterisation and raise a separate question from the kernel-space regime mechanism studied here. A full generative-forecasting comparison would require fixing rollout strategy, joint-versus-marginal scoring, sample budget, and inference-time compute, and is left as future work.

\subsection{Regime diagnostics}
\label{app:diagnostics}

The interpretability claim is supported visually in the main paper by gate trajectories $\pi_r(x,t)$ across the forecast horizon (Figures~\ref{fig:regime-diagnostics-main} and~\ref{fig:app-regime-remaining-six}), with the thresholded effective count $R_{\mathrm{eff}}=\sum_r\mathbf{1}\{\bar\pi_r>10^{-2}\}$ annotated on each panel and $\bar\pi_r=\E_\xi[\pi_r(\xi)]$ implicit in the colour ordering. The released codebase additionally produces per-regime parameter summaries of $(\tau_r,\nu_r)$ that identify low-noise/high-noise and light-tailed/heavy-tailed regimes, gate-entropy diagnostics, regime-conditional calibration plots that compute interval coverage after conditioning on the dominant regime, and residual-variance maps for $v_{\mathrm{res}}(x,t)$ showing where the finite regime mixture requires an additional variance floor. A same-backbone MDN can match a marginal predictive density without forming a stable allocation of inputs to reusable states, so these diagnostics together constitute the empirical evidence for the regime interpretation.

Figure~\ref{fig:app-regime-remaining-six} extends Figure~\ref{fig:regime-diagnostics-main} to the six datasets not shown in the main-text diagnostic panel. The examples were selected to cover different uncertainty phenomena rather than to maximise visual smoothness. ETTh1 shows dense overlapping residual states on a canonical benchmark, with the long horizon activating almost all $R_{\max}=16$ candidate regimes rather than forcing a sparse changepoint story. ETTh2 shows the same horizon-indexed mechanism on a cleaner thermal oscillation: the one-step gate cycles with the waveform, while the 24-step gate spreads mass across broader phase and amplitude states. ETTm2 illustrates a sharper timescale split, with a nearly single short-horizon regime but many long-horizon regimes once phase uncertainty accumulates. Aggregate Electricity is the sparse counterexample: a strong daily load cycle is explained using only a few active regimes, showing that the finite gate does not automatically invent a complicated state narrative. Nasdaq captures a macro-financial shock, where the short-horizon gate is comparatively stable through the COVID drawdown and recovery but the long-horizon gate reallocates mass around the crash and rebound. Weather gives a multichannel meteorological example; one-step temperature is nearly deterministic over the selected window, but the 24-step forecast uses many rapidly alternating regimes, consistent with front-driven uncertainty that appears only once the forecast horizon lengthens.

\begin{figure}[p]
\centering
\includegraphics[width=0.98\linewidth]{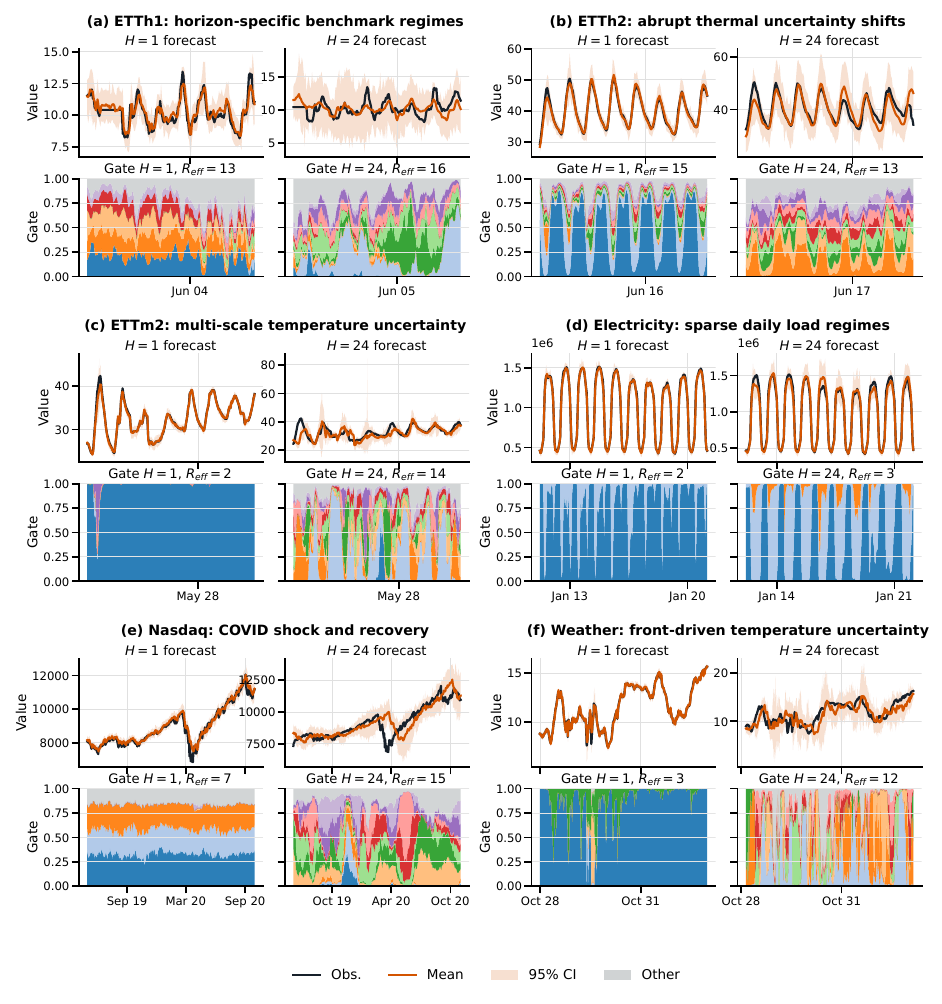}
\caption{Additional representative \DeRegime{} diagnostics for the six benchmark datasets not shown in Figure~\ref{fig:regime-diagnostics-main}. The panels show ETTh1, ETTh2, ETTm2, aggregate Electricity, Nasdaq, and Weather. Across datasets, the same diagnostic reveals sparse recurring states, dense overlapping residual states, and horizon-dependent growth in the number of active regimes. Gate colours are local display ranks by average mass within each panel; grey groups the remaining low-mass regimes as ``Other''.}
\label{fig:app-regime-remaining-six}
\end{figure}

\end{document}